\documentclass{article}

\usepackage{mlmath}
\usepackage[final]{neurips_2021}
\usepackage[utf8]{inputenc} % allow utf-8 input
\usepackage[T1]{fontenc}    % use 8-bit T1 fonts
\usepackage{hyperref}       % hyperlinks
\usepackage{url}            % simple URL typesetting
\usepackage{booktabs}       % professional-quality tables
\usepackage{amsfonts}       % blackboard math symbols
\usepackage{nicefrac}       % compact symbols for 1/2, etc.
\usepackage{microtype}      % microtypography
\usepackage{xcolor}         % colors

\usepackage{times}
\usepackage{graphicx}
\usepackage{subfigure}
\usepackage{caption}
\usepackage{amsmath,amsfonts,amssymb,mathtools,amsthm}
\newtheorem*{prop*}{Proposition}
\newtheorem{prop}{Proposition}
\newtheorem{rmk}{Remark}

\newtheorem{theorem}{Theorem}
\newtheorem*{theorem*}{Theorem}
\newtheorem{lemma}{Lemma}
\newtheorem*{lemma*}{Lemma}

\newtheorem*{property*}{Property}
\newtheorem*{remark}{Remark}
\newtheorem{definition}{Definition}

\newtheorem*{assumption*}{Assumption}

\newcommand{\rnarr}{\textnormal{narr}}
\newcommand{\rwide}{\textnormal{wide}}

\title{Embedding Principle of Loss Landscape of Deep Neural Networks}

\author{
Yaoyu Zhang\textsuperscript{\rm 1,2}\thanks{Corresponding author: zhyy.sjtu@sjtu.edu.cn.}, Zhongwang Zhang\textsuperscript{\rm 1} \thanks{Part of this work is done when ZZ was an undergraduate student of Zhiyuan Honors Program at Shanghai Jiao Tong University.},
Tao Luo\textsuperscript{\rm 1}, Zhi-Qin John Xu\textsuperscript{\rm 1}\thanks{Corresponding author: xuzhiqin@sjtu.edu.cn.} \\
\textsuperscript{\rm 1}  School of Mathematical Sciences, Institute of Natural Sciences, MOE-LSC and \\
Qing Yuan Research Institute, Shanghai Jiao Tong University \\
\textsuperscript{\rm 2} Shanghai Center for Brain Science and Brain-Inspired Technology\\
\{zhyy.sjtu, 0123zzw666, luotao41, xuzhiqin\}@sjtu.edu.cn.
}

\begin{document}

\maketitle

\begin{abstract}
    Understanding the structure of loss landscape of deep neural networks (DNNs) is obviously important. In this work, we prove an embedding principle that the loss landscape of a DNN ``contains'' all the critical points of all the narrower DNNs. More precisely, we propose a critical embedding such that any critical point, e.g., local or global minima, of a narrower DNN can be embedded to a critical point/affine subspace of the target DNN with higher degeneracy and preserving the DNN output function. Note that, given any training data, differentiable loss function and differentiable activation function, this embedding structure of critical points holds.
    This general structure of DNNs is starkly different from other nonconvex problems such as protein-folding.
     Empirically, we find that a wide DNN is often attracted by highly-degenerate critical points that are embedded from narrow DNNs. The embedding principle provides a new perspective to study the general easy optimization of wide DNNs and unravels a potential implicit low-complexity regularization during the training.
    Overall, our work provides a skeleton for the study of loss landscape of DNNs and its implication, by which a more exact and comprehensive understanding can be anticipated in the near future. 
\end{abstract}

% Introduction and Overview
\section{Introduction}
Understanding the loss landscape of DNNs is essential for a theory of deep learning. An important problem is to quantify exactly how the loss landscape looks like \citep{weinan2020towards}.
This problem is difficult since the loss landscape is so complicated that it can almost be any pattern \citep{skorokhodov2019loss}. Moreover, its high dimensionality and the dependence on data, model and loss  make it very difficult to obtain a general understanding through empirical study. Therefore, though it has been extensively studied over the years, it remains an open problem to provide a clear picture about the organization of its critical points and their properties.

In this work, we make a step towards this goal through proposing a very general embedding operation of network parameters from narrow to wide DNNs, by which
% We identify a relation of critical points of the loss landscape among DNNs with different widths which does not depend on training data and loss function and unravels very general embedding structure of critical points intrinsic to the layer-wise architecture of DNNs. 
% Our main result is concluded by the 
we prove an embedding principle for fully-connected DNNs stated \emph{intuitively} as follows:

\emph{\textbf{Embedding principle}: the loss landscape of any network ``contains'' all critical points of all narrower networks.}

% \textbf{Embedding principle}: any network ``contains'' all critical points of narrower networks.

A ``narrower network'' means a DNN of the same depth but width of each layer no larger than the target DNN. The embedding principle slightly abuses the notion of ``contain'' since parameter space of DNNs of different widths are different. However, this inclusion relation is reasonable in the sense that, by our embedding operation, any critical point of any narrower network can be embedded to a critical point of the target network preserving its output function. Because of this criticality preserving property, we call this embedding operation the critical embedding.

We conclude our study by a ``principle'' since the embedding principle is a very general property of loss landscape of DNNs independent of the training data and choice of loss function, and is intrinsic to the layer-wise architecture of DNNs. 
In addition, the embedding principle is closely related to the training of DNNs. For example, as shown in Fig. \ref{fig:syntraining}(a), the training of a width-$500$ two-layer tanh NN experiences stagnation around the blue dot presumably very close to a saddle point, where the loss decreases extremely slowly. As shown in Fig. \ref{fig:syntraining}(b), we find that the DNN output at this blue point (red solid) is very close to the output of the global minimum (black dashed) of the width-$1$ NN, indicating that the underlying two critical points of two DNNs with different widths have the same output function conforming with the embedding principle. Importantly, this example shows that the training of a wide DNN can indeed experience those critical points from a narrow DNN unraveled by the embedding principle. Moreover, it demonstrates the potential of a transition from a local/global minimum of a narrow NN to a saddle point of a wide NN, which may be the reason underlying the easy optimization of wide NNs.

The embedding principle suggests an underlying mechanism to understand why heavily overparameterized DNNs often generalize well \citep{breiman1995reflections,zhang2016understanding} as follows. Roughly, the overparameterized DNN has a large capacity, which seems contradictory to the conventional learning theory, i.e., learning by a model of large capacity easily leads to overfitting. The embedding principle shows that the optima of a wide network intrinsically may be embedded from an optima of a much narrower network, thus, its effective capacity is much smaller. For example, as illustrated in Fig. \ref{fig:syntraining}, training of a heavily overparametrized width-$500$ NN (vs. $50$ training data) with small initialization first stagnated around a saddle presumably from width-$1$ NN and later converges to a global minimum presumably from width-$3$ NN, which clearly does not overfit.
This implicit regularization effect unraveled by the embedding principle is consistent with previous works, such as low-complexity bias \citep{arpit2017closer,kalimeris2019sgd,jin2020quantifying}, low-frequency bias \citep{xu_training_2018,xu2019frequency,rahaman2018spectral}, and condensation phenomenon of network weights \citep{luo2021phase,chizat_global_2018,ma2020quenching}.

\begin{figure}[h]
	\centering
	\subfigure[]{\includegraphics[width=0.24\textwidth]{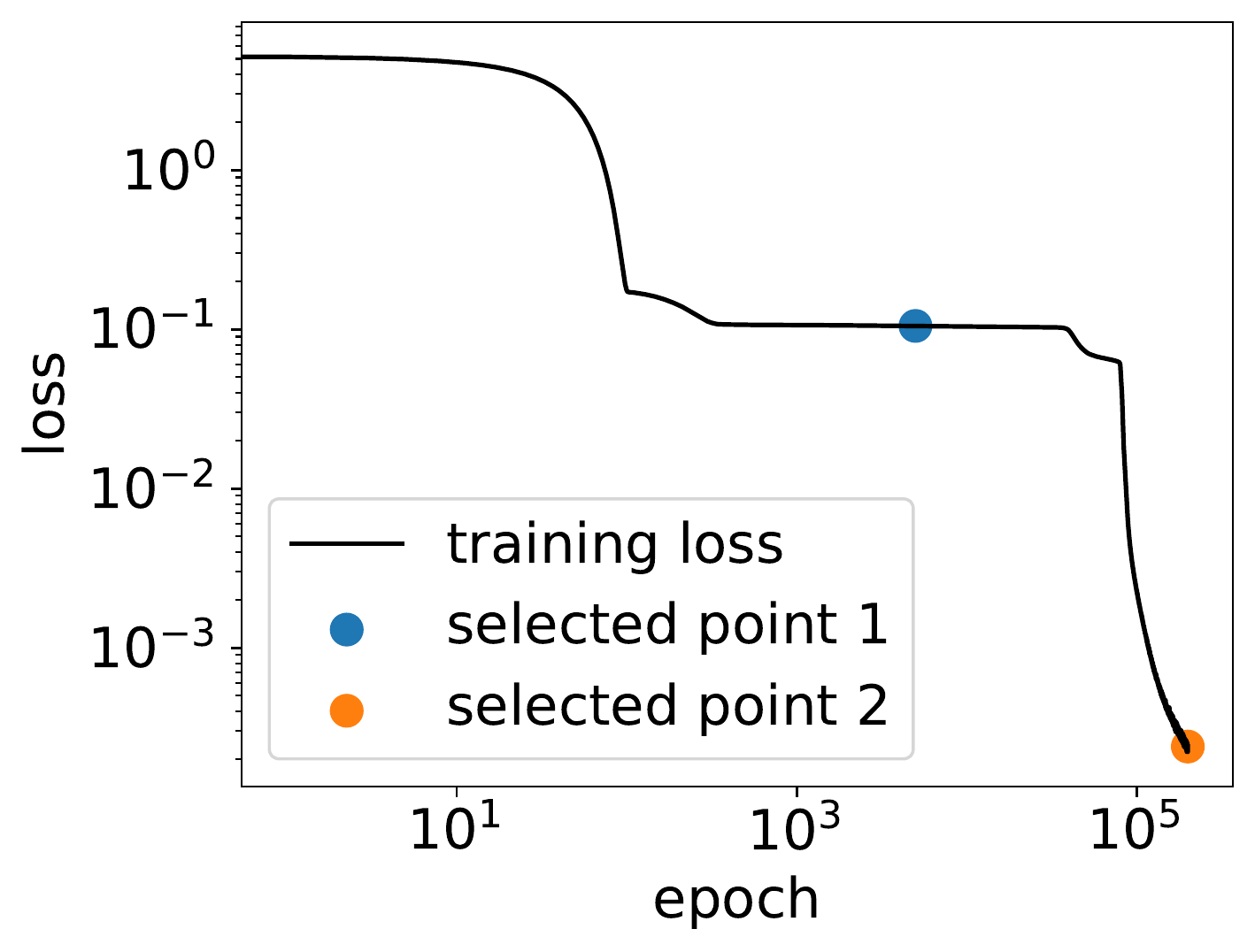}}
	\subfigure[]{\includegraphics[width=0.24\textwidth]{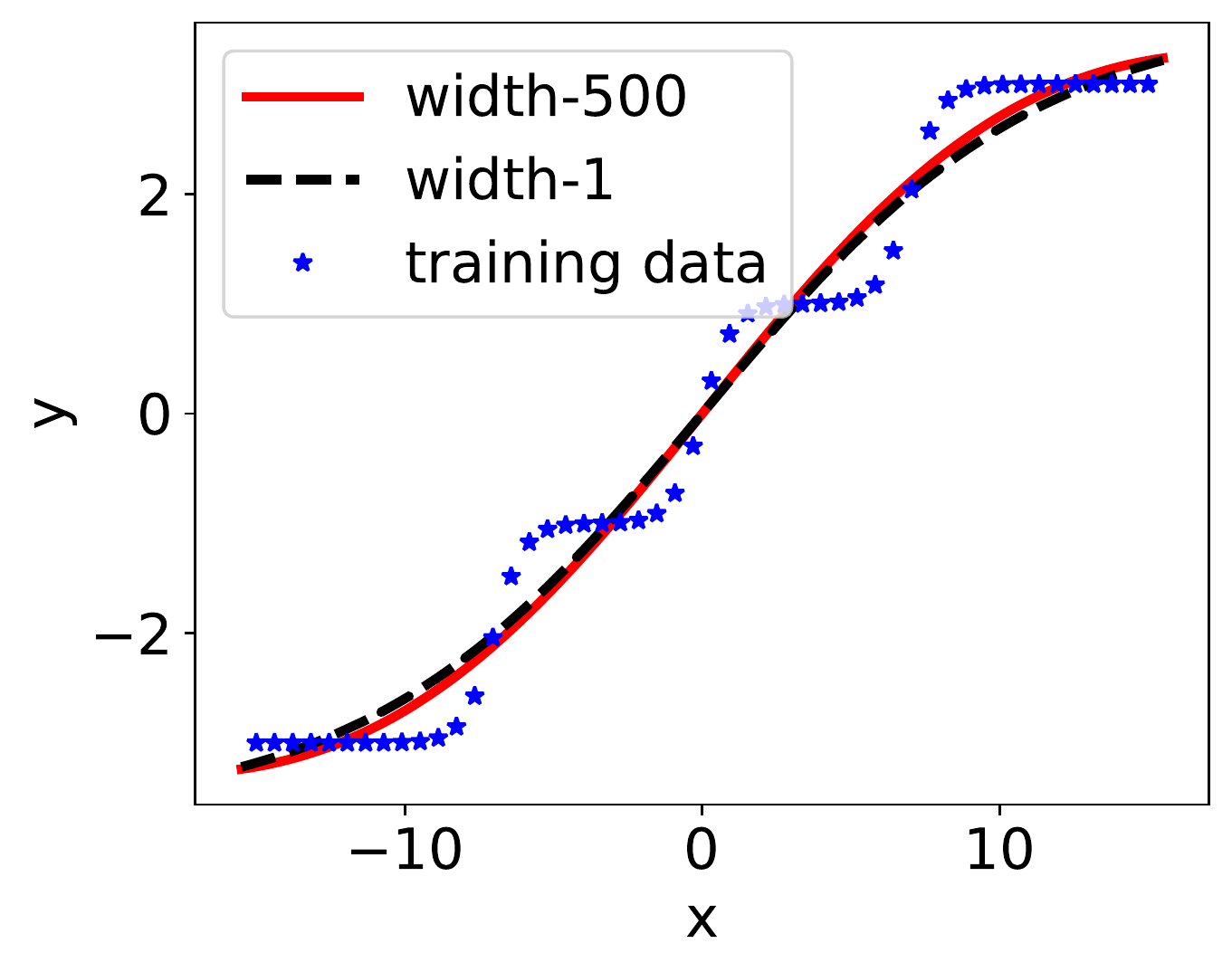}}
	\subfigure[]{\includegraphics[width=0.24\textwidth]{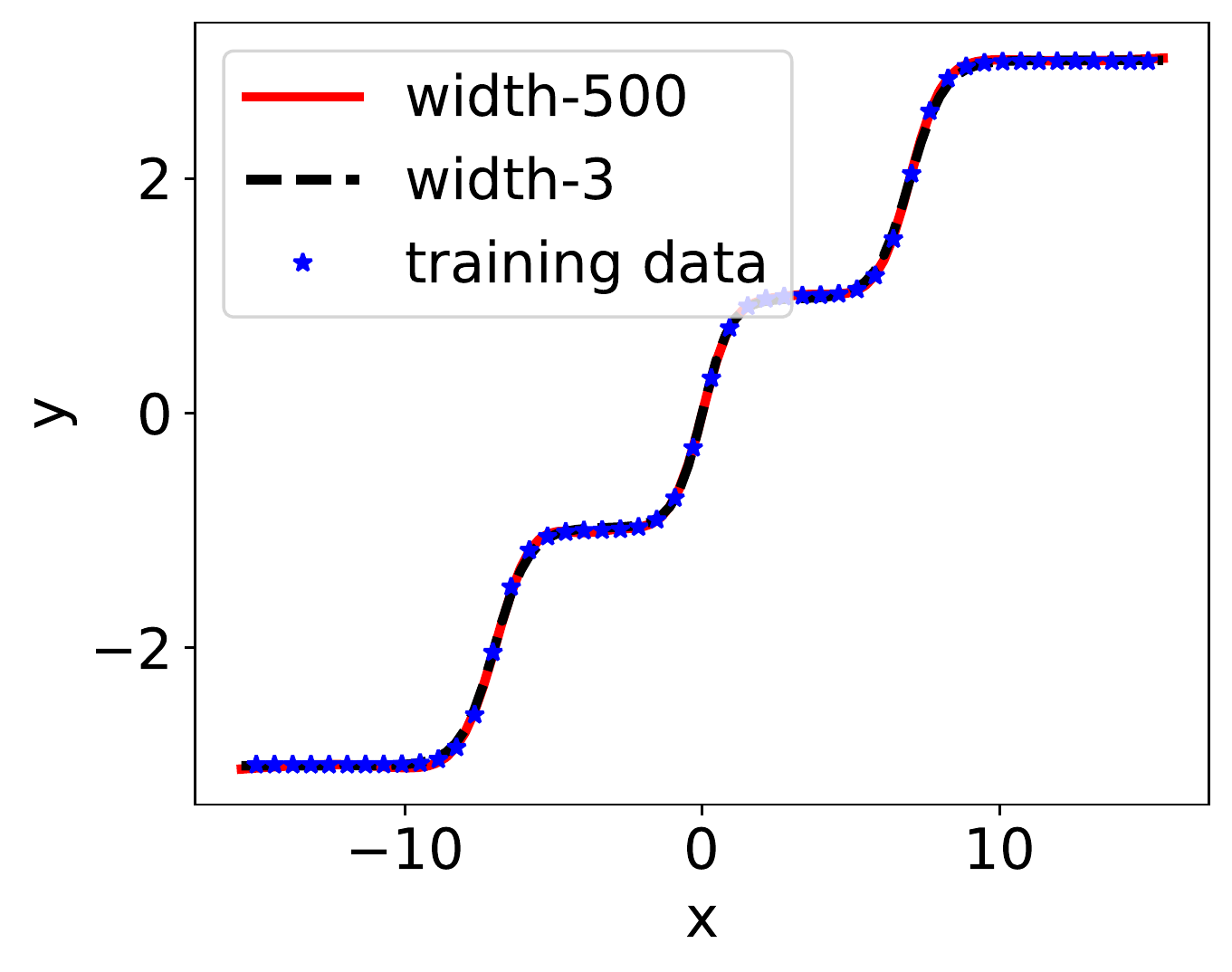}}
	\caption{(a) The training loss of two-layer tanh neural network with $500$ hidden neurons. (b) (c) Red solid: the DNN output at a training step indicated by (b) the blue dot or (c) the orange dot in (a); Black dashed: the output of the global minimum of (b) width-$1$ DNN or (c) width-$3$ DNN, respectively; Blue dots: training data. \label{fig:syntraining}}
\end{figure}

\section{Related works}
The loss landscape of DNNs is complex and related to the generalization. \cite{skorokhodov2019loss} numerically show that the loss landscape can almost be any pattern. \cite{keskar2017large} visualize minimizers in a 1d slice and suggest that a flat minimizer generalizes better.  \cite{wu2017towards} find that the volume of basin of attraction of good minima may dominate over that of poor minima in practical problems. \cite{he2019asymmetric} show that  at a local minimum there exist many asymmetric directions such that the loss increases abruptly along one side, and slowly along the opposite side.

Degeneracy is also an important property of minima. \cite{cooper2018loss}
shows that global minima is typically a high dimensional manifold for overparameterized DNNs.
\cite{sagun2016singularity} empirically shows that Hessian of the minimizer obtained by the training has many zero eigenvalues. Under strong assumptions, \cite{choromanska2015loss} shows  minima tend to be highly degenerate. This work demonstrates wide existence of highly degenerate critical points, including local or global minima and saddle points, in the loss landscape by the embedding principle.   

Lots of previous theoretical works focus on very wide DNNs, such as the phase diagram of two-layer ReLU infinite-width NNs \citep{luo2021phase}, NTK regime \citep{jacot_neural_2018,arora2019exact,zhang_type_2019,du2019gradient,zou2018stochastic,allen2019convergence,E2019comparative}, mean-field regime  ~\citep{mei_mean_2018,rotskoff_parameters_2018,chizat_global_2018,sirignano_mean_2020}. By the embedding principle, this work demonstrate the loss landscape similarity between a moderate-width NN and a very wide NN, that they share a set of critical points embedded from that of narrower NNs. Therefore, results about infinite-width NNs could provide valuable insights about training of finite-width NNs used in practice.

% \cite{luo2021phase} establishs a phase diagram of three regimes at the infinite-width limit for the two-layer ReLU NNs, i.e., linear regime, critical regime and condensed regime. In the linear regime, such as the example of neural tangent kernel case \citep{jacot2018neural}, the learning converges to the global minima exponentially \citep{luo2021phase}. 
% These previous works focus on the the infinite-width limit, while this work shows an embedding principle that connects a large set of critical points among NNs with different widths.

% For specific types of activation functions, such as linear function \citep{kawaguchi2016deep} and quadratic functions \citep{soltanolkotabi2018theoretical,du2018power,mannelli2020optimization}, [?discuss] we have better understanding about the loss landscape and the global optima can be achieved under mild assumptions. In contrast, embedding principle in this work is very general, which applies to general activation functions. 

The complexity of NN output increases during the training \citep{arpit2017closer,xu_training_2018,xu2019frequency,rahaman2018spectral,kalimeris2019sgd,goldt2020modeling,he2020assessing,mingard2019neural,jin2020quantifying}.
% , which suggest the NN finds a function of as low complexity as it can be to interpolate the training data. 
For example, the frequency principle \citep{xu_training_2018,xu2019frequency}  states that DNNs often fit target
functions from low to high frequencies during the training. 

In \cite{zhang2021embedding}, we make a comprehensive extension of this conference paper. In the long paper, we provide a mathematical definition of the critical embedding and propose a new class of general compatible embeddings, which is a much wider class of critical embeddings than composition embeddings in this work. These general compatible embeddings provide much richer details about the geometry of critical submanifolds of DNN loss landscape. Note that the composition embedding technique is also studied in \cite{fukumizu2019semi} and \cite{csimcsek2021geometry}.
% This works hints an implicit regularization effect that the training path with small initialization may find a minima that is embedded from a narrow network. Therefore, for a very wide NN, complexity of its learned output function at a minimum may be effectively very low.

\section{Main results}
In this section, we intuitively summarize our key theoretical results about the embedding principle and empirically demonstrate its relevance to practice, starting from proposing an embedding operation as follows. Rigorous theoretical description and proofs are presented in the latter sections. 

% For the loss landscape of any target network, any critical point of a narrower network can be embedded to a critical point of this target network with representation preserving. 

\subsection{Characteristics of embedding principle}\label{sec:characteristics}

Consider a neural network $\vf_{\vtheta}(\vx)$, where $\vtheta$ is the set of all network parameters, $\vx\in\sR^{d}$  is the input. We summarize assumptions and provide definitions needed for all our results in this work below.
\begin{assumption*}

(i) $L$-layer ($L\geq 2$) fully-connected NN.

(ii) Training data $S=\{(\vx_i,\vy_i)\}_{i=1}^n$, $n\in\sZ^+\cup\{+\infty\}$. 

(iii)Loss function $\RS(\vtheta)=\Exp_S\ell(\vf_{\vtheta}(\vx),\vy)$.

(iv) Loss function and activation function are differentiable. Note that, even for functions like ReLU or hinge loss, as long as we uniquely assign a subgradient to their non-differentiable points, all our results still hold.
\end{assumption*}
% The loss for the training data $S=\{(\vx_{i},\vf^{*}(\vx_{i})\}_{i=1}^{n}$ is denoted by $\RS(\vtheta)=\frac{1}{n}\sum_{i=1}^{n}\ell(\vf_{\vtheta}(\vx_i),\vf^{*}(\vx_{i}))$. 

% ADD Some conditions for $S=\{(\vx_{i},f^{*}(\vx_{i})\}_{i=1}^{n}$ if there is any. 

\begin{definition}[\textbf{critical point}] 
    % Given the neural network architecture and the training data set $S$, 
    Parameter vector $\vtheta$ is a critical point of the landscape of $\RS$ if $\nabla_{\vtheta} \RS(\vtheta) =\vzero$.
\end{definition}
\begin{definition}[\textbf{critical submanifold/affine subspace}] 
    A critical submanifold or affine subspace $\fM$ is a connected subsubmanifold or affine subspace of the parameter space $\sR^M$, such that each $\vtheta\in\fM$ is a critical point of loss with the same loss value.
\end{definition}

\begin{definition}[\textbf{degree of degeneracy}] 
    % Given the neural network architecture and the training data set $S$, 
 The degree of degeneracy of point $\vtheta$ in the landscape of $\RS$ is the corank of Hessian matrix $\nabla_{\vtheta}\nabla_{\vtheta} \RS$, i.e., number of the zero eigenvalues. 
\end{definition}
\begin{remark}

In the above definition of degree of degeneracy, we require twice differentiable activation function and twice differentiable loss to compute Hessian for convenience. For loss and activation functions with only first-order differentiability, we extend the definition of degree of degeneracy as follows: for any critical point $\vtheta$ belonging to a $K$-dimensional critical submanifold $\fM$, its degree of degeneracy is at least $K$.
\end{remark}

We first introduce one-step embedding intuitively, and leave the rigorous definition latter. 
As shown in Fig.~\ref{fig:onestep}, an one-step embedding is performed by splitting any hidden neuron, say the black neuron in the left network, into two neurons colored in blue and purple in the right network. The input weights of the two splitted neurons are the same as the input weights of the original black neuron. Each output weight of the original black neuron is splitted into two parts of fraction $\alpha$ and $(1-\alpha)$ ($\alpha\in\sR$, a hyperparameter), respectively.  The multi-step embedding is the composition of multiple one-step embeddings. Since each one-step embedding can add one neuron to a selected layer, parameter of any NN can be embedded to the parameter space of any wider NN through a multi-step embedding. The multi-step embedding operation leads to the following property readily.

\begin{figure}
    \centering
    \includegraphics[width=0.8\textwidth]{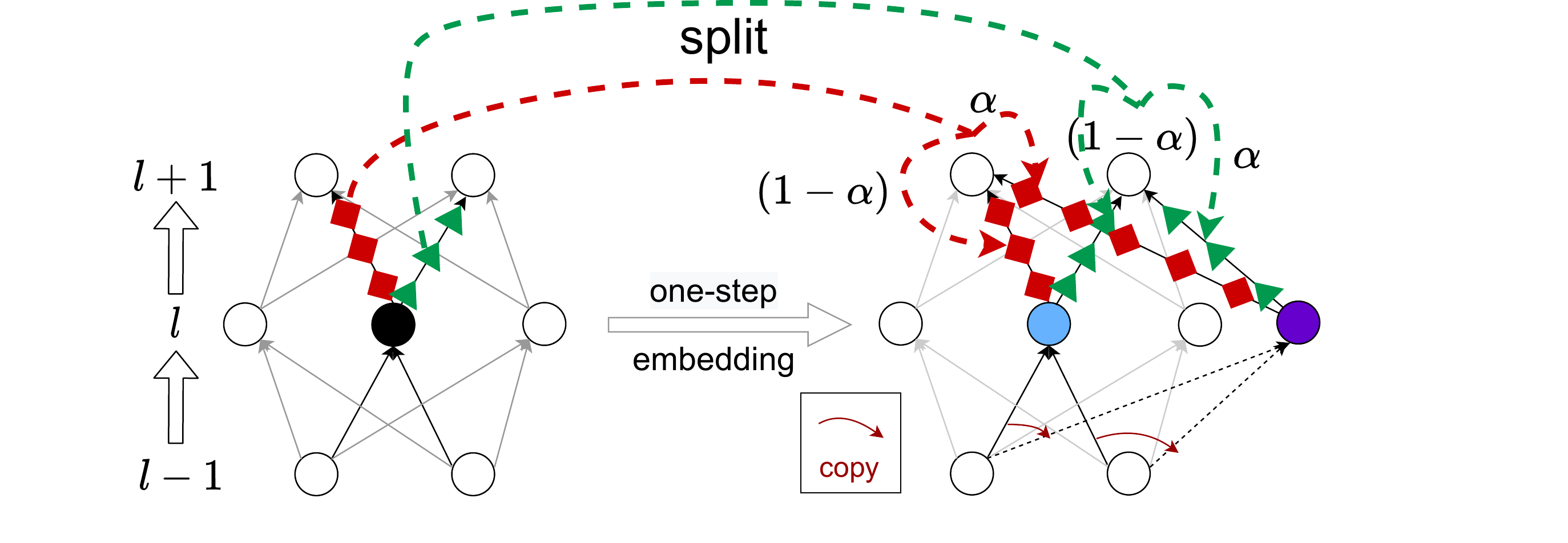} 
    \label{fig:onestep}
\caption{Illustration of one-step embedding. The black neuron in the left network is splitted into the blue and purple neurons in the right network. The red (green) output weight of the black neuron in the left net is splitted into two red (green) weights in the right net with ratio $\alpha$ and $(1-\alpha)$, respectively.    \label{fig:onestep}}
\end{figure}

\begin{prop*}[\textbf{one-step embedding preserves network properties, informal Prop. \ref{prop:one-step-embed}}] 
% \begin{prop}[\textbf{representation preserving}]
For any point $\vtheta_{\rnarr}$ of a DNN, a point $\vtheta_{\rwide}$ of a wider DNN obtained from $\vtheta_{\rnarr}$ by one-step embedding satisfies\\
(i) $\vf_{\vtheta_{\rnarr}}(\vx) = \vf_{\vtheta_{\rwide}}(\vx)$ for any $\vx$; \\
(ii) representation of the wide DNN at $\vtheta_{\rwide}$, i.e., the set of all different response functions of neurons, is the same as representation of the narrow DNN at $\vtheta_{\rnarr}$.
\end{prop*}
The most important property of this embedding is criticality preserving as follows.
\begin{theorem*}[\textbf{criticality preserving, informal Theorem \ref{lem:crit-prese}}]
    For any critical point $\vtheta_{\rnarr}$ of a DNN, a point $\vtheta_{\rwide}$ of a wider DNN obtained from $\vtheta_{\rnarr}$ by multi-step embedding is a critical point.
    % For any critical point $\vtheta_{\rnarr}$ of a DNN and a point $\vtheta_{\rwide}$ of a wider DNN, $\vtheta_{\rwide}$ is obtained from $\vtheta_{\rnarr}$ by multi-step embedding, then,  $\vtheta^{c}_{w}$ is a critical point.
\end{theorem*}
The embedding operation explains the cause of a type of degeneracy in the loss landscape.
\begin{theorem*}[\textbf{degeneracy of embedded critical points, informal Theorem \ref{lem:multi-criti}}]
    If output weights of neurons in each layer of a DNN at a critical point $\vtheta_{\rnarr}$ are not all zero, then, for any $K$-neuron wider DNN, $\vtheta_{\rnarr}$ can be embedded to a $K$-dimensional critical affine subspace.
    % For any neuron in a critical point $\vtheta_{\rnarr}$  of a DNN, if its output weight is non-zero, $\vtheta_{\rnarr}$ can be embedded to a $K$-dimensional affine subspace ($K\in \sR$), which belongs to a sub-manifold of critical points of a $K$-neuron wider DNN.
% For a DNN of width $\{m_l\}_l$, a critical point of critical index $\{m_l'\}_l$ ($m_l'\leqslant m_l$ for any $l$) has degeneracy of at least $\sum_{l=1}^{L-1}{(m_l-m_l')}$.
\end{theorem*}

% \begin{remark}
% For $\vtheta_{\rwide}$, obtained by critical point $\vtheta_{\rnarr}$ of a narrow network through $K$-step embedding, $\nabla_{\vtheta}^{2} \RD(\vtheta_{\rwide})$ has at least $K$ more zero eigenvalues than $\nabla_{\vtheta}^{2} \RD(\vtheta_{\rnarr})$.
% \end{remark}

\begin{remark}
By above theorem, each step of embedding of a critical point in general is accompanied by an increased degree of degeneracy. Therefore, degenerate critical points in general widely exist in the loss landscape of a DNN, and non-degenerate critical points are rare because they often become degenerate once embedded to a wider DNN.
% for any $\vtheta_{\rwide}$ obtained through $K$-step embedding from critical point $\vtheta_{\rnarr}$ of a narrow network, its degree pf degeneracy is at least $K$ higher than $\vtheta_{\rnarr}$.
\end{remark}

In previous studies, degeneracy is often considered as a consequence of over-parameterization depending on the size of training data $n$. Specifically, \cite{cooper2018loss} proves that the degree of degeneracy of global minima is $m-n$ for $1$-d output, where $m$ is the number of network parameters. However, we demonstrate by the above theorem that regardless of whether the NN is over-parameterized, degenerate critical points are prevalent in its loss landscape as long as narrower DNNs possess critical points.

% \textcolor{red}{However, we demonstrate by the above theorem that having degenerate critical points is a key characteristic of the loss landscape of NNs regardless of whether a NN is overparameterized or not.}

% the number of the zero eigenvalues of $\nabla_{\vtheta}^{2} \RD(\vtheta)$
% (ii) the degree of degeneracy at $\vtheta^{c}_{w}(\vx)$ is at least one larger than that of  $\vtheta^{c}_{n}(\vx)$; \\ 

% Note that result (i) and (ii) are non-trivial, while (iii) and (iv) can be obtained readily from the embedding rule.

\subsection{Numerical experiments}\label{sec:exp}
\paragraph{Experimental setup.}
Throughout this work, we use two-layer fully-connected neural network with size $d$-$m$-$d_{out}$. The input dimension $d$ is determined by the training data. The output dimension $d_{out}$ is different for different experiments. The number of hidden neurons $m$ is specified in each experiment. All parameters are initialized by a Gaussian distribution with mean zero and variance specified in each experiment. We use MSE loss trained by full batch gradient descent for 1D fitting problems (Figs. \ref{fig:syntraining}, \ref{fig:expand}(a) and \ref{fig:embedloss}), and default Adam optimizer with full batch for others. The learning rate is fixed throughout the training. More details of experiments are shown in Appendix B. 

\paragraph{Increment of degeneracy through embedding.} We train a two-layer NN of width $m_{\rm small}=2$ to learn data of Fig. \ref{fig:syntraining} shown in Fig. \ref{fig:expand}(a) or Iris dataset \citep{fisher1936use} in Fig. \ref{fig:expand}(b) to a critical point. 
We first roughly estimate the possible interval of critical points by observing where the loss decays very slowly, and then take the point with the smallest derivative of the parameters (use $L_1$ norm) as an empirical critical point. The $L_1$ norm of the derivative of loss function at the empirical critical point is approximately $7.15\times 10^{-15}$ for Fig. \ref{fig:expand}(a) and $3.72\times 10^{-13}$ for Fig. \ref{fig:expand}(b), which are reasonably small.
We then embed this critical point to networks of width $m=3$ and $m=4$ through an one-step or a two-step embedding, respectively. It is obvious from Fig. \ref{fig:expand} that each step of embedding incurs one more zero eigenvalue in the Hessian matrix,
% (machine accuracy $\sim 10^{-16}$)
which conforms with Theorem \ref{lem:multi-criti}. Moreover, in Fig. \ref{fig:expand}(a), for $m=2$, all eigenvalues are positive (red) indicating the critical point obtained by training is a local or global minimum. After embedding, this point becomes a saddle due to the emergence of negative eigenvalues (blue). Specifically, in both Fig. \ref{fig:expand}(a) and (b), we observe a steady increase of significant negative eigenvalues, e.g., from $0$ to $1$ to $2$ in (a) and from $3$ to $5$ to $7$ in (b), which implies reduced difficulty in escaping from the corresponding critical point in a wider NN during the training.

\begin{figure}[h]
	\centering
% 	\subfigure[synthetic data, $m_{\rm small}=1$]{\includegraphics[width=0.45\textwidth]{pic/expand/m1eig.pdf}}
	\subfigure[synthetic data]{\includegraphics[width=0.34\textwidth]{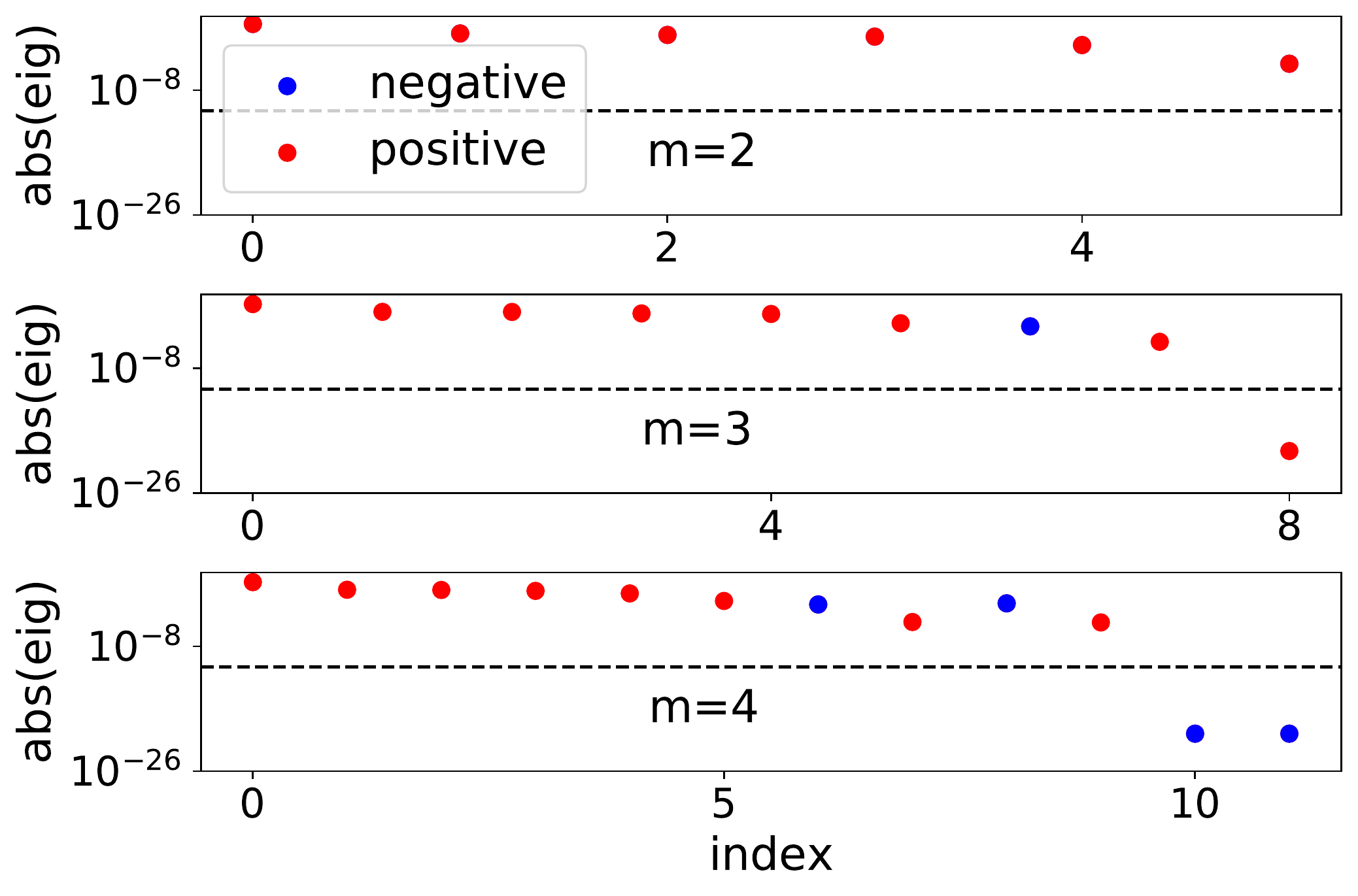}}
% 	\subfigure[synthetic data, $m_{\rm small}=3$]{\includegraphics[width=0.45\textwidth]{pic/expand/m3eig.pdf}} 
	\subfigure[Iris data]{\includegraphics[width=0.34\textwidth]{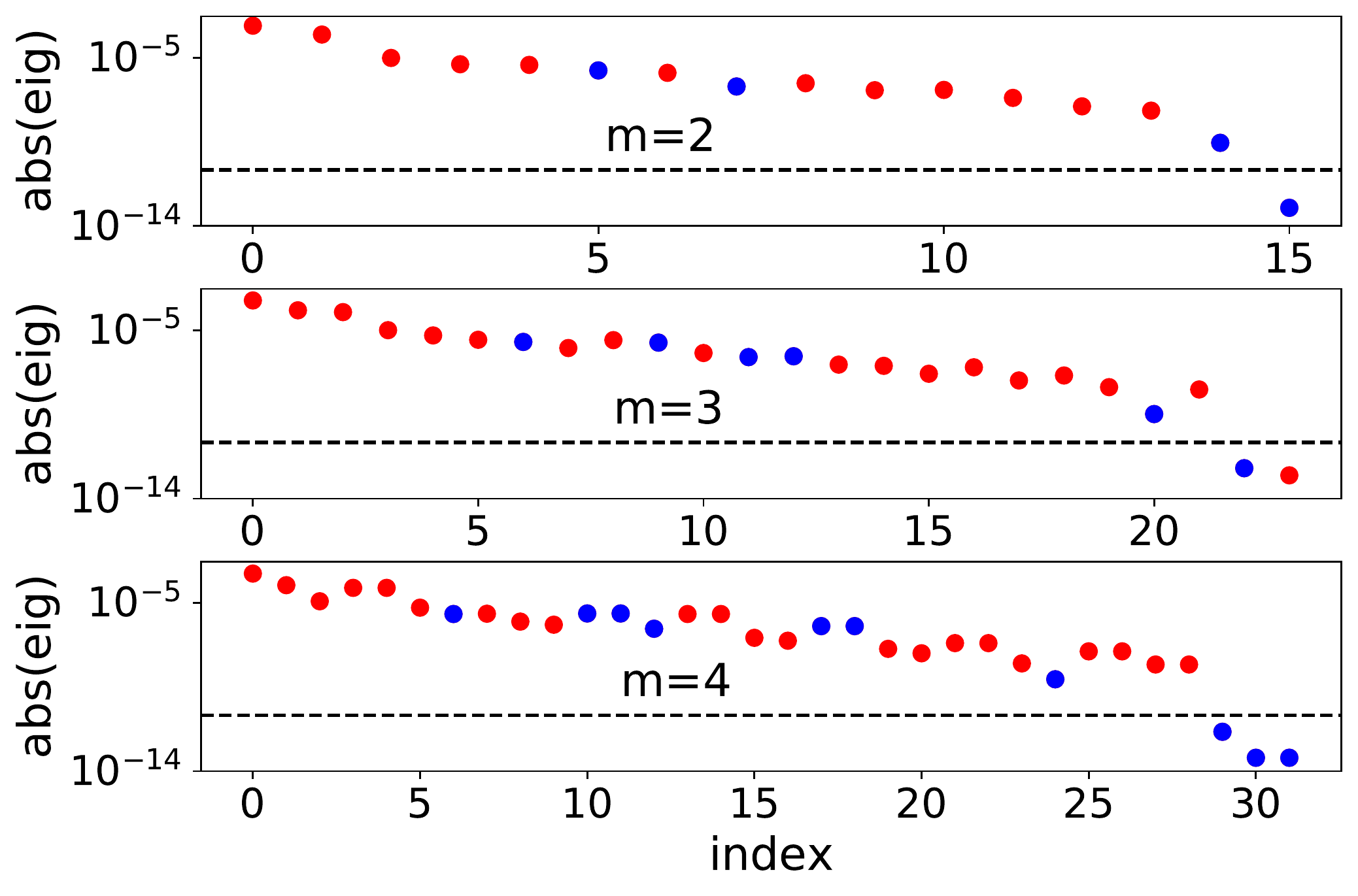}} 
% 	\caption{Eigenvalues of Hessian of NNs at the critical points embedded from the NN with width $m_{small}=2$  for learning data of Fig. \ref{fig:syntraining} in (a) and for Iris dataset \citep{fisher1936use} in (b). The text $m$ in each sub-figure is the NN width after embedding. Before embedding, i.e., $m=m_{\rm small}$, eigenvalues of all cases are positive (red) and non-zero, while one-step embedding incurs one more negative eigenvalue and one close-zero eigenvalues.	\label{fig:expand}}
    \caption{Eigenvalues of Hessian of NNs at the critical points embedded from the NN with width $m_{\rm small}=2$  for learning data of Fig. \ref{fig:syntraining} in (a) and for Iris dataset in (b). The value of $m$ in each sub-figure is the NN width after embedding. The auxiliary dash line in each sub-figure is $y=10^{-11}$.
    We equally split one neuron of a width-$2$ two-layer NN at a critical point into k neurons ($k=2,3$), whose input weights remain the same but output weights are $1/k$ of the original neuron.	\label{fig:expand}}
\end{figure} 

\paragraph{Empirical diagram of loss landscape.}  In Fig. \ref{fig:embedloss}, we present an empirical diagram of loss landscape of a width-$3$ two-layer tanh DNN to visualize a set of its critical points predicted by the embedding principle, i.e., critical points embedded from network of width-$1$ or -$2$ respectively as well as critical points that cannot be obtained through embedding. Through the training of width-$1$, -$2$, -$3$ network respectively on the training data presented in Fig. \ref{fig:syntraining} for multiple trials, we discover $1$ critical point for width-$1$ network, $2$ critical points for width-$2$ network and $1$ critical point for width-$3$ network that cannot be embedded from a narrower NN. Then, embedding all these four critical points to critical points/affine subspaces of loss landscape of the width-$3$ network, we obtain four sets of critical points with their loss values, output functions, degrees of degeneracy and width of network they embedded from illustrated in Fig. \ref{fig:embedloss}.
% They are not all the critical points of the DNN model, but the critical points reached during the training process. To verify this, we train the network 200 times for different initial parameters, and then use the trained network parameters as initial values, and use the ``fsolve'' function to solve the critical points. After excluding the special case (the output function is always 0), we found there are only three cases. In fact, the ``fsolve'' function is just to make the model closer to the critical point, without making major changes to the network parameters. 
This diagram immediately tells us what attracts the gradient-based training trajectory for a width-$3$ network. Specifically, if stagnation happens during the training, this diagram informs us the potential loss values and output functions at stagnation, which could help us better understand the nonlinear training process of not only a width-$3$ network but also much wider networks due to the embedding principle.  Furthermore, as illustrated in Fig. \ref{fig:syntraining} for the training process of a $500$-neuron NN, 
saddle points of a wide NN, effectively local or global minima of narrow NNs, composes a trajectory, which may serve as a compass for achieving a global minimum from narrow NNs of low complexity.

% through four critical points when the training dataset is the one in Fig. \ref{fig:syntraining}. For branch $m_0 = 1$, we train a two-layer network with width $m_0 = 1$ to a global minimum, and then embed the global minimum to a critical point of network with width $m=3$, therefore, $m_0$ is the effective size. The procedure is similar for branch $m_0 = 2$, while we embed the global minimum by selecting each of two neurons leading to two end nodes. No embedding is performed for branch $m_0 = 3$. As the number of the effective neurons increases, the DNN output is closer to the training data, so as the decreasing of the loss. As shown in the training process of Fig. \ref{fig:syntraining}, these saddle points of a large network, effective global minima of narrow networks, compose a trajectory, which may serve as a compass for achieving a global minimum that can fit the training data but also has least effective capacity.

\begin{figure}[h]
	\centering
	\includegraphics[width=0.6\textwidth]{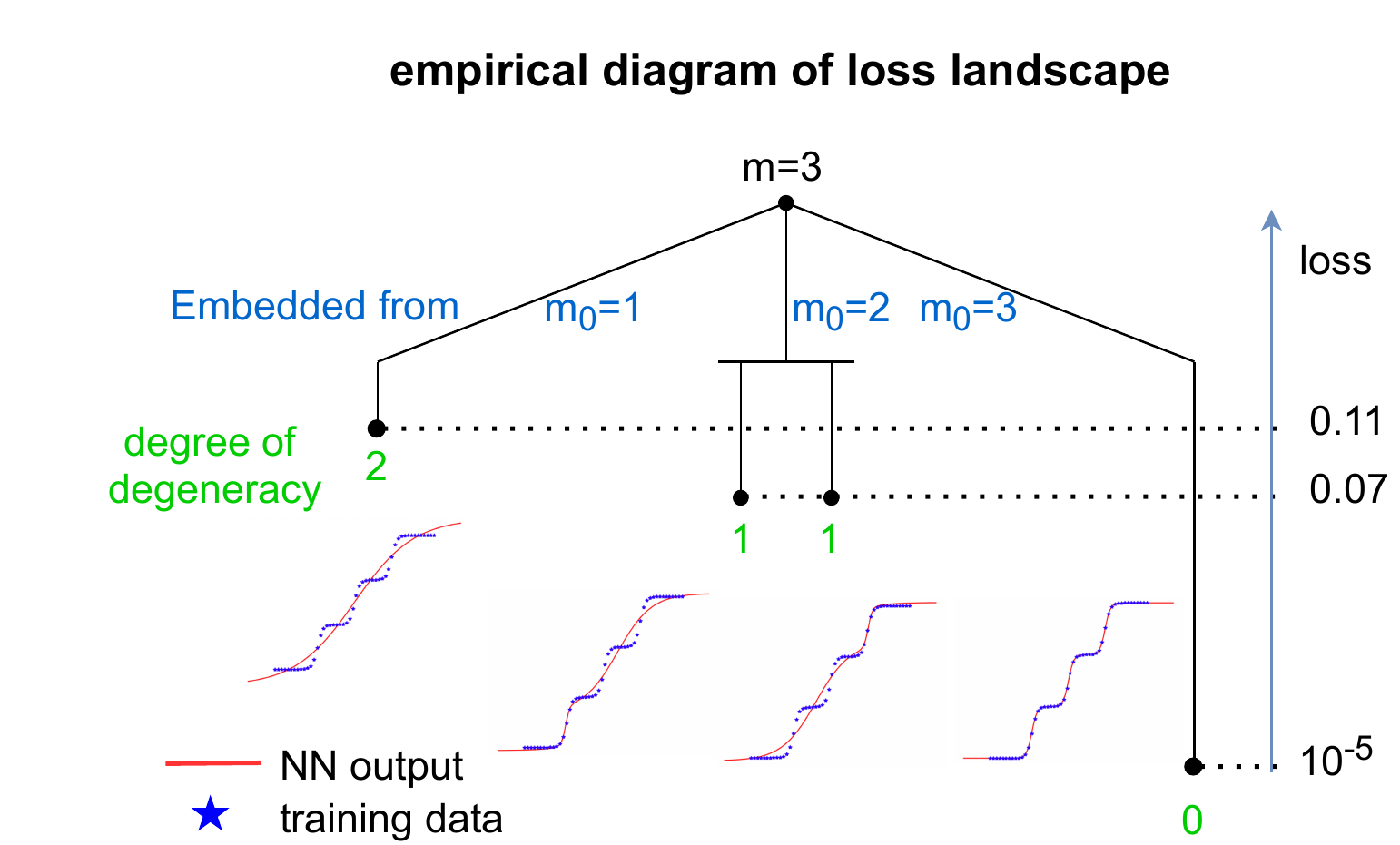}
	\caption{Empirical diagram of loss landscape of a width-$3$ two-layer tanh NN, i.e., all critical points width-3 or narrower NNs may get close to during the training under proper initialization. Each black dot at terminal represents a specific set of critical points of loss embedded from critical points of NNs of different widths (blue). These critical points have different loss values (ordinate), degrees of degeneracy (green) and output functions (red solid curves) as labelled in the figure. The blue dots represent the training data. We use the same equal splitting as Fig. \ref{fig:expand} to embed critical points of width-1 or width-2 NN to critical points of the width-3 NN and compute the hessian to obtain the corresponding degree of degeneracy. Note that the degree of degeneracy of these critical points computed numerically in this problem coincides with their minimal degree of degeneracy $m-m_0$ in Theorem \ref{lem:multi-criti}.  \label{fig:embedloss}}
\end{figure} 

\paragraph{Reduction of DNN at critical point.} The embedding principle predicts a class of critical points of a NN embedded from much narrower NNs. At such a critical point, we shall be able to find neuron groups, within which neurons have similar orientation of input weights presumably originated from the same neuron of a narrow NN through embedding. This prediction is confirmed by the following experiment in Fig. \ref{fig:expand_mnist}. We train a width-$400$ two-layer ReLU NN $\vf_{\vtheta}=\sum_{k=1}^m a_k\sigma(\vw_k^T\tilde{\vx})$\ ($\tilde{\vx} = [\vx^\T,1]^\T$) on $1000$ training samples of the MNIST dataset with small initialization. At the blue dot in Fig. \ref{fig:expand_mnist}(a), the loss decreases very slowly, presumably very close to a saddle point. We then examine the orientation similarity between each pair of neuron input weights by computing the inner product of two normalized input weight. 
% In the example in Fig. \ref{fig:expand_mnist}, we show a two-layer sigmoid DNN when learning MNIST dataset experiences critical points that are critical points of narrower networks. At the blue dot in Fig. \ref{fig:expand_mnist}(a), the loss decreases very slowly, it is then reasonable to presume the DNN is very close a saddle point. We then examine the orientation similarity between each pair of neuron input weights by computing the inner product of two normalized input weight. 
As shown in Fig. \ref{fig:expand_mnist}(b), there emerge $58$ groups of neurons (neurons with very small amplitudes are neglected and later directly removed), where similarity between input weights in the same group is at least $0.9$. For each group $S_{\rm similar}$, we randomly select a neuron $j$, replace its output weight by $\sum_{k\in S_{\rm similar}} a_k \|\vw_k\|_2 /\|\vw_j\|_2$, and discard all other neurons in the group. The parameter set before reduction is denoted by $\vtheta_{\rm ori}$, and after reduction by $\vtheta_{\rm redu}$. Width of the NN is reduced from $400$ to $58$. We train the reduced NN from $\vtheta_{\rm redu}$ as shown in Fig. \ref{fig:expand_mnist}, which stagnated after a few steps at the same loss value as the blue point in Fig. \ref{fig:expand_mnist}(a) marked by the blue dash and represented by the blue point in Fig. \ref{fig:expand_mnist}(c). We then compare the prediction between original model and the reduced model at the corresponding blue points on $10000$ test data as shown in Fig.\ref{fig:expand_mnist}(d). For each grid, color indicates the frequency of that prediction pair. 
Specifically, the highlight of diagonal element indicts high prediction agreement of two models (overall $\sim98.5\%$). Therefore, this critical point of the reduced width-$58$ NN well matches the critical point of the original width-$400$ NN, clearly demonstrating the relevance of our embedding principle to real dataset training.
% The square distance, denoed by $D_{s}$, between the original network and the reduced one for each input is defined as $D_{s}=\|f_{\vtheta_{\rm ori}}(\vx)-f_{\vtheta_{\rm redu}}(\vx)\|_2^2.$  The difference between original network and the reduced one is very small as indicated by the histogram of the $D_{s}$ over the training dataset in Fig.\ref{fig:expand_mnist}(d). Therefore, we have shown the embedding principle in learning real dataset.[?discuss]

% \begin{figure}[h]
% 	\centering
% 	\includegraphics[width=0.9\textwidth]{pic/expand/m2eig_iris.png}
% 	\caption{Eigenvalues of Hessian of NNs at the critical points embedded from the NN with width $m_{small}=2$. The dataset we used is iris \citep{fisher1936use}. The text $m$ in sub-figure is the NN width after embedding. Before embedding, i.e., $m=m_{\rm small}$, network has one close-zero eigenvalue, while one-step embedding incurs one more close-zero eigenvalue.	\label{fig:expand_iris}} 
% \end{figure} 

\begin{figure}[h]
	\centering
	\subfigure[initial loss]{\includegraphics[width=0.23\textwidth]{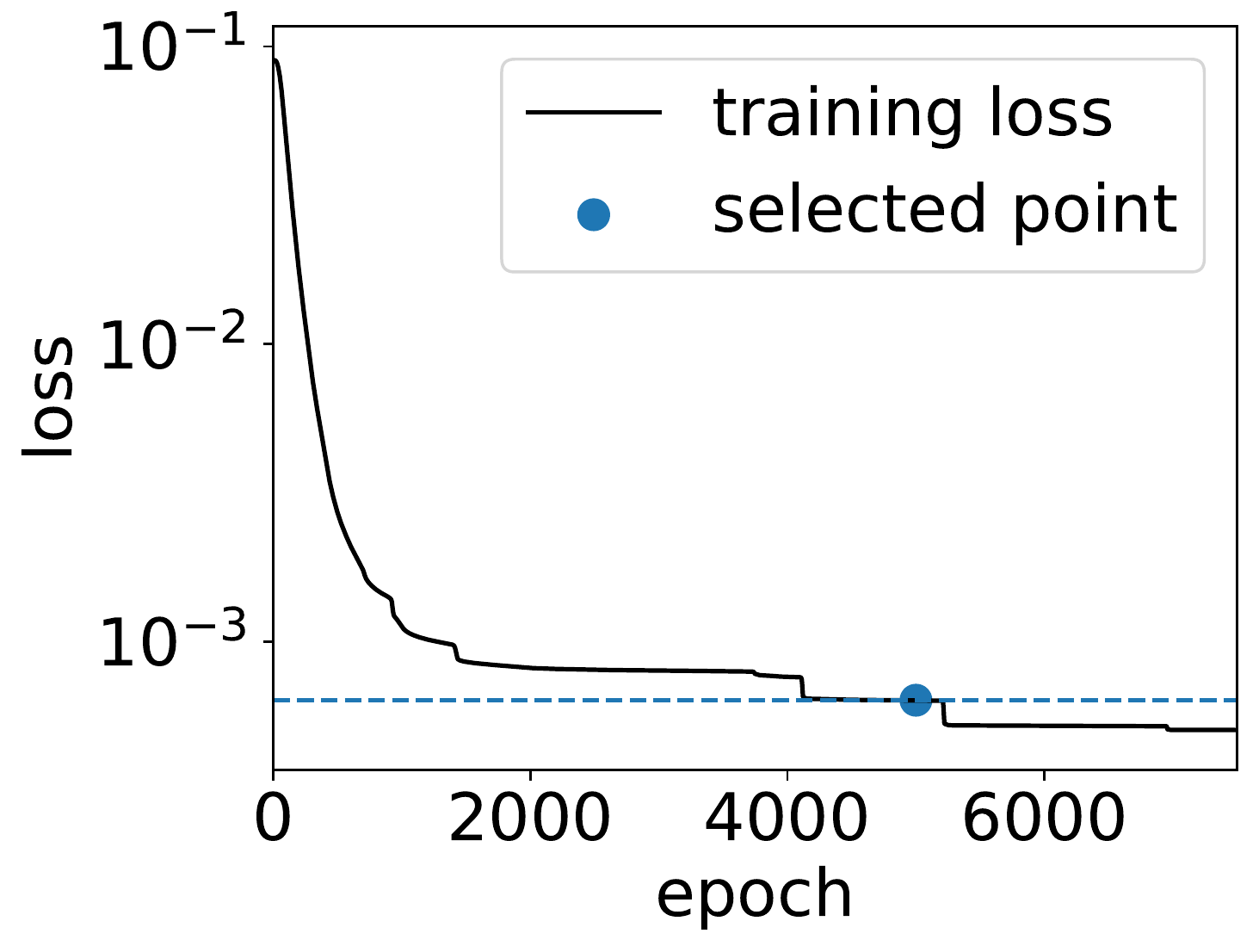}}
	\subfigure[orientation similarity ]{\includegraphics[width=0.23\textwidth]{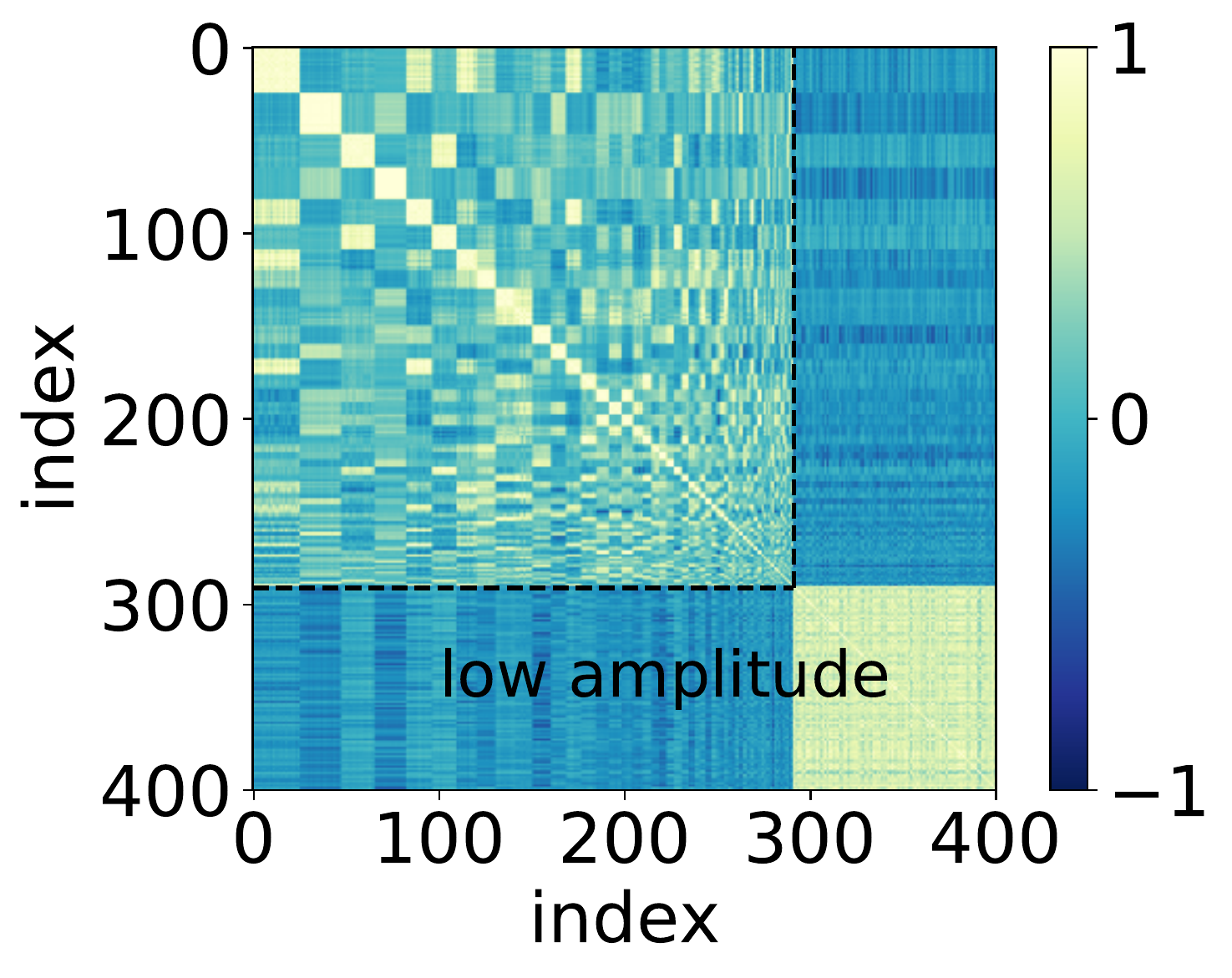}}
	\subfigure[loss after reduction]{\includegraphics[width=0.23\textwidth]{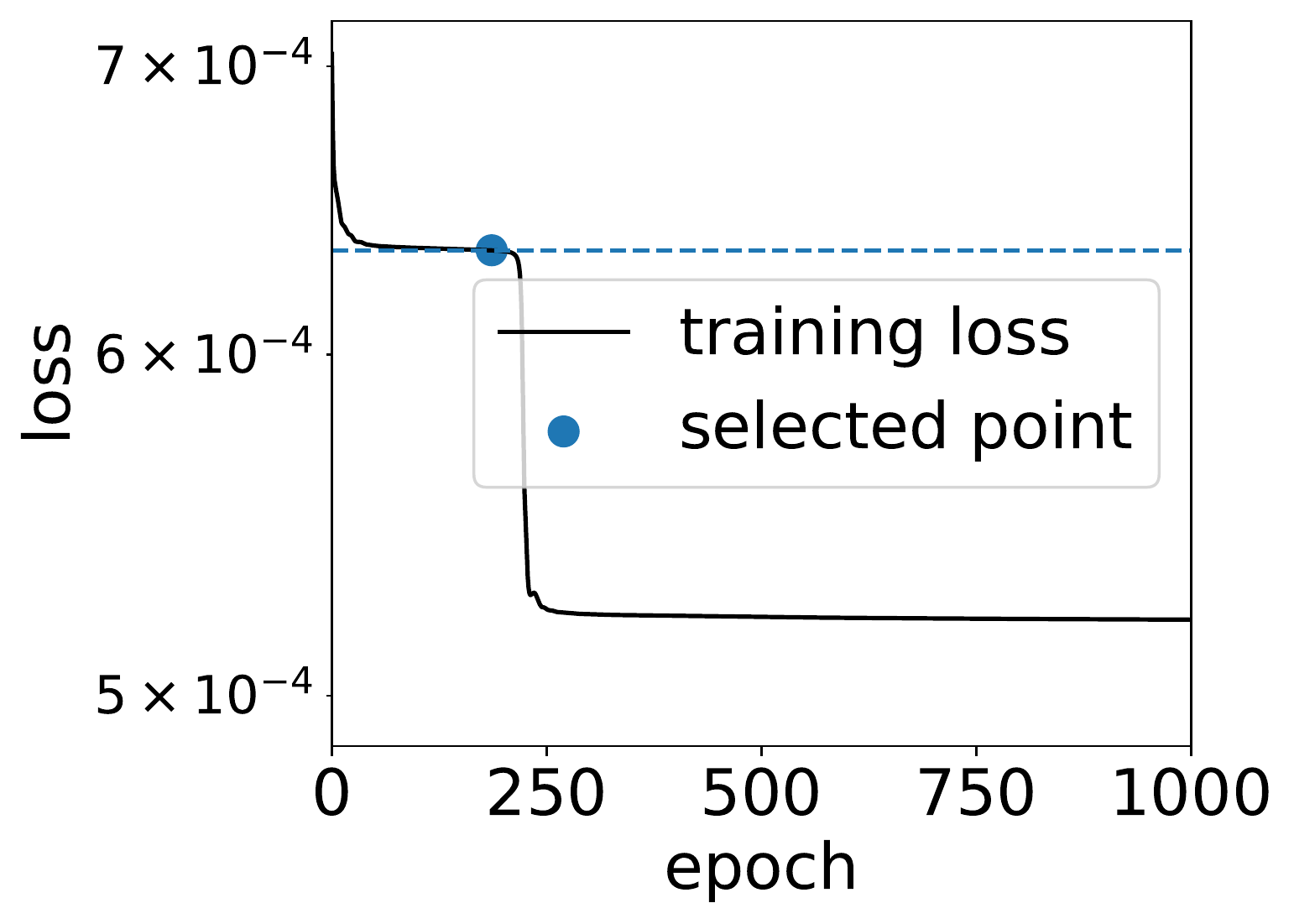}}
	\subfigure[prediction similarity]{\includegraphics[width=0.23\textwidth]{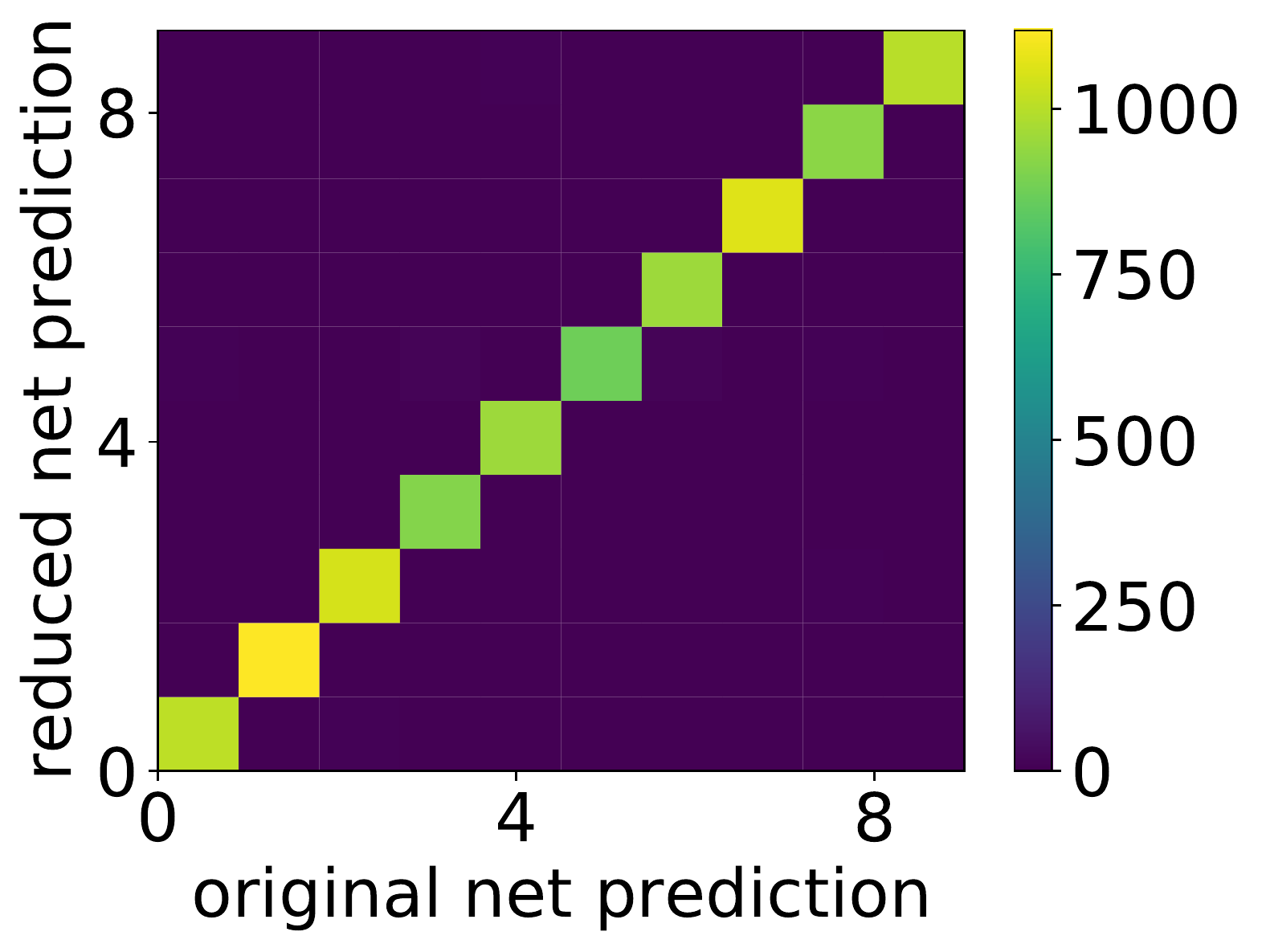}}
	\caption{(a) The training loss of the initial network on MNIST. The blue point is selected for reduction. (b) The normalized inner product of input weights for different neurons. The abscissa and ordinate represent neuron index. Neurons in ``low amplitude'' region has much lower amplitude than others, hence are removed. (c) The training loss of the reduced network. Blue dash indicates the same loss value as the blue dash in (a). The blue point is selected as a representative for comparison.  (d) Prediction similarity. For each grid, color indicates the frequency of that prediction pair. 	\label{fig:expand_mnist}}
\end{figure}

\section{Preliminaries}

\subsection{Deep Neural Networks}
Consider $L$-layer ($L\geq 2$) fully-connected DNNs with a general differentiable activation function. We regard the input as the $0$-th layer and the output as the $L$-th layer. Let $m_l$ be the number of neurons in the $l$-th layer. In particular, $m_0=d$ and $m_L=d'$. For any $i,k\in \sN$ and $i<k$, we denote $[i:k]=\{i,i+1,\ldots,k\}$. In particular, we denote $[k]:=\{1,2,\ldots,k\}$.
Given weights $W^{[l]}\in \sR^{m_l\times m_{l-1}}$ and bias $b^{[l]}\in\sR^{m_{l}}$ for $l\in[L]$, we define the collection of parameters $\vtheta$ as a $2L$-tuple (an ordered list of $2L$ elements) whose elements are matrices or vectors
\begin{equation}
    \vtheta=\Big(\vtheta|_1,\cdots,\vtheta|_L\Big)=\Big(\mW^{[1]},\vb^{[1]},\ldots,\mW^{[L]},\vb^{[L]}\Big).
\end{equation}
where the $l$-th layer parameters of $\vtheta$ is the ordered pair $\vtheta|_{l}=\Big(\mW^{[l]},\vb^{[l]}\Big),\quad l\in[L]$.
We may misuse of notation and identify $\vtheta$ with its vectorization $\mathrm{vec}(\vtheta)\in \sR^M$ with $M=\sum_{l=0}^{L-1}(m_l+1) m_{l+1}$.

Given $\vtheta\in \sR^M$, the neural network function $\vf_{\vtheta}(\cdot)$ is defined recursively. First, we write $\vf^{[0]}_{\vtheta}(\vx)=\vx$ for all $\vx\in\sR^d$. Then for $l\in[L-1]$, $\vf^{[l]}_{\vtheta}$ is defined recursively as 
$\vf^{[l]}_{\vtheta}(\vx)=\sigma (\mW^{[l]} \vf^{[l-1]}_{\vtheta}(\vx)+\vb^{[l]})$.
Finally, we denote
\begin{equation}
    \vf_{\vtheta}(\vx)=\vf(\vx,\vtheta)=\vf^{[L]}_{\vtheta}(\vx)=\mW^{[L]} \vf^{[L-1]}_{\vtheta}(\vx)+\vb^{[L]}.
\end{equation}
For notational simplicity, we may drop the subscript $\vtheta$ in $\vf^{[l]}_{\vtheta}$, $l\in[0:L]$.

We introduce the following notions for the convenience of the presentation in this paper. 
\begin{definition}[\textbf{Wider/narrower DNN}]
    We write $\mathrm{NN}(\{m_l\}_{l=0}^{L})$ for a fully-connected neural network with width $(m_0,\ldots,m_L)$.
    Given two $L$-layer ($L\geq 2$) fully-connected neural networks $\mathrm{NN}(\{m_l\}_{l=0}^{L})$ and $\mathrm{NN}'(\{m'_l\}_{l=0}^{L})$, if $m'_0=m_0$, $m'_L=m_L$, and for any $l\in[L-1]$, $m'_l\geq m_l$ and $K=\sum_{l=1}^{L-1}(m'_l-m_l)\in\sN_+$, then we say that $\mathrm{NN}'(\{m'_l\}_{l=0}^{L})$ is $K$-neuron wider than $\mathrm{NN}(\{m_l\}_{l=0}^{L})$ and $\mathrm{NN}(\{m_l\}_{l=0}^{L})$ $K$-neuron narrower than $\mathrm{NN}'(\{m'_l\}_{l=0}^{L})$. 
\end{definition}
% \begin{definition}[\textbf{Silent neuron}]
%     Given a $L$-layer ($L\geq 2$) fully-connected neural network with width $(m_0,\ldots,m_{L})$ and network parameters
%     $\vtheta=(\mW^{[1]},\vb^{[1]},\cdots,\mW^{[L]},\vb^{[L]})\in\sR^M$,
%     if there exists $l\in[L-1]$ and $s\in[m_l]$ such that $\mW^{[l+1]}_{[1:m_{l+1}],s}= \mzero_{m_{l+1}\times 1}$, then we say that the $s$-th neuron of the $l$-th layer is a silent neuron.
% \end{definition}

\subsection{Loss function}
The training data set is denoted as  $S=\{(\vx_i,\vy_i)\}_{i=1}^n$, where $\vx_i\in\sR^d$, $\vy_i\in \sR^{d'}$. For simplicity, here we assume an unknown function $\vy$ satisfying $\vy(\vx_i)=\vy_i$ for $i\in[n]$. The empirical risk reads as
\begin{equation}
    \RS(\vtheta)=\frac{1}{n}\sum_{i=1}^n\ell(\vf(\vx_i,\vtheta),\vy(\vx_i))=\Exp_S\ell(\vf(\vx,\vtheta),\vy).
\end{equation}
where the expectation $\Exp_S h(\vx):=\frac{1}{n}\sum_{i=1}^n h(\vx_i)$ for any function $h:\sR^d\to \sR$ and the loss function $\ell(\cdot,\cdot)$ is differentiable and the derivative of $\ell$ with respect to its first argument is denoted by $\nabla\ell(\vy,\vy^*)$.  Generally, we always take derivatives/gradients of $\ell$ in its first argument with respect to any parameter.
We consider gradient flow of $\RS$ as the training dynamics, i.e., $\D \vtheta/\D t = -\nabla_{\vtheta} \RS(\vtheta)$ with $\vtheta(0) = \vtheta_0$.

% \begin{equation}
%     \left\{
%     \begin{aligned}
%         \dfrac{\D \vtheta}{\D t}
%          & = -\nabla_{\vtheta} \RS(\vtheta), \\
%         \vtheta(0)
%          & = \vtheta_0.
%     \end{aligned}
%     \right.\label{eq..TrainingL2Loss}
% \end{equation}
We define the error vectors $\vz_{\vtheta}^{[l]}=\nabla_{\vf^{[l]}}\ell$ for $l\in[L]$ and the feature gradients $\vg_{\vtheta}^{[L]}=\mathbf{1}$ and $\vg^{[l]}_{\vtheta} =\sigma^{(1)}\left(\mW^{[l]} \vf^{[l-1]}_{\vtheta}+\vb^{[l]}\right)$ for $l\in[L-1]$. Here $\sigma^{(1)}$ is the first derivative of $\sigma$. We call $\vf^{[l]}_{\vtheta}$, $l\in[L]$ feature vectors.
The collections of feature vectors, feature gradients, and error vectors are $\vF_{\vtheta}= \{\vf^{[l]}_{\vtheta}\}_{l=1}^L,
    \vG_{\vtheta}
    = \{\vg^{[l]}_{\vtheta}\}_{l=1}^L,
    \vZ_{\vtheta}
    = \{\vz^{[l]}_{\vtheta}\}_{l=1}^L.$ 
Using backpropagation, we can calculate the gradients as follows
\begin{align*}
    \vz_{\vtheta}^{[L]}
    &=\nabla\ell,\quad
    \vz_{\vtheta}^{[l]}
    = (\mW^{[l+1]})^\T\vz_{\vtheta}^{[l+1]}\circ\vg_{\vtheta}^{[l+1]},\quad l\in[L-1],\\
    \nabla_{\mW^{[l]}}\ell
    &= \vz_{\vtheta}^{[l]}\circ\vg_{\vtheta}^{[l]}(\vf_{\vtheta}^{[l-1]})^\T,\quad
    \nabla_{\vb^{[l]}}\ell
    = \vz_{\vtheta}^{[l]}\circ\vg_{\vtheta}^{[l]},\quad l\in[L].
\end{align*}
Here we use $\circ$ for the Hadamard product of two matrices of the same dimension.

\section{Critical embedding}\label{sec:theory}

We introduce the one-step embedding for the DNNs which will lead us to general embeddings.
% define several mappings to establish mathematical results for the embedding principle. 

\begin{definition}[\textbf{one-step embedding}]
    Given a $L$-layer ($L\geq 2$) fully-connected neural network with width $(m_0,\ldots,m_{L})$ and network parameters
    $\vtheta=(\mW^{[1]},\vb^{[1]},\cdots,\mW^{[L]},\vb^{[L]})\in\sR^M$,
    for any $l\in[L-1]$ and any $s\in[m_l]$, we define the linear operators $\fT_{l,s}$ and $\fV_{l,s}$ applying on $\vtheta$ as follows
    \begin{align*}
        \fT_{l,s}(\vtheta)|_k
        &=\vtheta|_k,\quad k\neq l,l+1,\\
        \fT_{l,s}(\vtheta)|_l
        &= \left(\left[ {\begin{array}{cc}
        \mW^{[l]} \\
        \mW^{[l]}_{s,[1:m_{l-1}]} \\
        \end{array} } \right],
        \left[ {\begin{array}{cc}
        \vb^{[l]} \\
        \vb^{[l]}_s \\
        \end{array} } \right]\right),\quad
        \fT_{l,s}(\vtheta)|_{l+1}
        = \left(
        \left[\mW^{[l+1]},\mzero_{m_{l+1}\times 1}\right],
        \vb^{[l+1]}\right),\\
        \fV_{l,s}(\vtheta)|_k
        &=\left(\mzero_{m_{k}\times m_{k-1}},\mzero_{m_{k}\times 1}\right),\quad k\neq l,l+1,\\
        \fV_{l,s}(\vtheta)|_l
        &=\left(\mzero_{(m_{l}+1)\times m_{l-1}},\mzero_{(m_{l}+1)\times 1}\right),\\
        \fV_{l,s}(\vtheta)|_{l+1}
        &= \left(
        \left[\mzero_{m_{l+1}\times (s-1)},-\mW^{[l+1]}_{[1:m_{l+1}],s},\mzero_{m_{l+1}\times (m_{l}-s)},\mW^{[l+1]}_{[1:m_{l+1}],s}\right],\mzero_{m_{l+1}\times 1}\right).
    \end{align*}
     Then the one-step embedding operator $\fT_{l,s}^{\alpha}$ is defined as for any $\vtheta\in\sR^M$ 
    \begin{equation*}
        \fT_{l,s}^{\alpha}(\vtheta)=(\fT_{l,s}+\alpha\fV_{l,s})(\vtheta).
    \end{equation*}
    Note that the resulting parameter $\fT_{l,s}^{\alpha}(\vtheta)$ corresponds to a $L$-layer fully-connected neural network with width $(m_0,\ldots,m_{l-1},m_l+1,m_{l+1},\ldots,m_{L})$.
\end{definition}
An illustration of $\fT_{l,s}$, $\fV_{l,s}$, and $\fT_{l,s}^{\alpha}$ can be found in Fig. S1 in Appendix.

\begin{lemma} \label{lem:repinv}
    Given a $L$-layer ($L\geq 2$) fully-connected neural network with width $(m_0,\ldots,m_{L})$, for any network parameters    $\vtheta=(\mW^{[1]},\vb^{[1]},\cdots,\mW^{[L]},\vb^{[L]})$ and for any $l\in[L-1]$, $s\in[m_l]$, we have the expressions for $\vtheta':=\fT_{l,s}^{\alpha}(\vtheta)$ \\
    (i) feature vectors in $\vF_{\vtheta'}$: $\vf^{[l']}_{\vtheta'}=\vf^{[l']}_{\vtheta}$, $l'\neq l$ and $\vf^{[l]}_{\vtheta'}=\left[(\vf_{\vtheta}^{[l]})^\T,(\vf_{\vtheta}^{[l]})_s\right]^\T$;
    
    (ii) feature gradients in $\vG_{\vtheta'}$: $\vg^{[l']}_{\vtheta'}=\vg^{[l']}_{\vtheta}$, $l'\neq l$ and $\vg^{[l]}_{\vtheta'}=
    \left[(\vg^{[l]}_{\vtheta})^\T,(\vg^{[l]}_{\vtheta})_s\right]^\T$;
    
    (iii) error vectors in $\vZ_{\vtheta'}$: 
    $\vz^{[l']}_{\vtheta'}=\vz^{[l']}_{\vtheta}$, $l'\neq l$\\
    and $\vz^{[l]}_{\vtheta'}=
    \left[ (\vz_{\vtheta}^{[l]})^\T_{[1:s-1]},(1-\alpha)(\vz_{\vtheta}^{[l]})_s, (\vz_{\vtheta}^{[l]})^\T_{[s+1:m_l]},\alpha (\vz_{\vtheta}^{[l]})_s \right]^\T$.
\end{lemma}
 An illustration of $\vF_{\vtheta}$ and $\vZ_{\vtheta}$ can be found in Fig. S2 in Appendix.

\begin{prop}[\textbf{one-step embedding preserves network properties}] \label{prop:one-step-embed}
    Given a $L$-layer ($L\geq 2$) fully-connected neural network with width $(m_0,\ldots,m_{L})$, for any network parameters $\vtheta=(\mW^{[1]},\vb^{[1]},\cdots,\mW^{[L]},\vb^{[L]})$ and for any $l\in[L-1]$, $s\in[m_l]$, the following network properties are preserved for $\vtheta'=\fT_{l,s}^{\alpha}(\vtheta)$:
    
    (i) output function is preserved: $f_{\vtheta'}(\vx)=f_{\vtheta}(\vx)$ for all $\vx$;
    
    (ii) empirical risk is preserved: $\RS(\vtheta')=\RS(\vtheta)$;
    
    (iii) the sets of features are preserved:  $\left\{\left(\vf^{[l]}_{\vtheta'}\right)_i\right\}_{i\in[m_{l}+1]}=\left\{\left(\vf^{[l]}_{\vtheta}\right)_i\right\}_{i\in[m_{l}]}$ and\\
    $\left\{\left(\vf^{[l']}_{\vtheta'}\right)_i\right\}_{i\in[m_{l'}]}=\left\{\left(\vf^{[l']}_{\vtheta}\right)_i\right\}_{i\in[m_{l'}]}$ for $l'\in[L]\backslash\{l\}$;
    
    % (iv) criticality is preserved: if $\nabla_{\vtheta}\RS(\vtheta)=\mzero$, then $\nabla_{\vtheta}\RS(\vtheta')=\mzero$.
\end{prop}

% \begin{proof}
%     The properties (i)--(iii) are direct consequences of Lemma 1. The criticality preserving property as the most important property of critical embedding is separately stated and proved below.
% \end{proof}

\begin{theorem}[\textbf{criticality preserving}]\label{lem:crit-prese}
    Given a $L$-layer ($L\geq 2$) fully-connected neural network with width $(m_0,\ldots,m_{L})$, for any network parameters    $\vtheta=(\mW^{[1]},\vb^{[1]},\cdots,\mW^{[L]},\vb^{[L]})$ and for any $l\in[L-1]$, $s\in[m_l]$, if $\nabla_{\vtheta}\RS(\vtheta)=\mzero$, then $\nabla_{\vtheta}\RS(\vtheta')=\mzero$.
\end{theorem}

\begin{lemma}[\textbf{increment of the degree of degeneracy}] \label{lem:degen}
    Given a $L$-layer ($L\geq 2$) fully-connected neural network with width $(m_0,\ldots,m_{L})$, if there exists $l\in[L-1]$, $s\in[m_l]$, and a $q$-dimensional manifold $\fM$ consisting of critical points of $\RS$ such that for any $\vtheta\in \fM$, $\mW^{[l+1]}_{[1:m_{l+1}],s}\neq \mzero$, then $\fM':=\{\fT_{l,s}^{\alpha}(\vtheta)|\vtheta\in\fM,\alpha\in \sR\}$ is a $(q+1)$-dimensional manifold consists of critical points for the corresponding $L$-layer fully-connected neural network with width $(m_0,\ldots,m_{l-1},m_l+1,m_{l+1},\ldots,m_{L})$.
\end{lemma}

% \begin{proof}
%     For any $\vtheta\in \fM$, let $\{\ve_i(\vtheta)\}_{i=1}^d$ be a basis of its tangent space $T_{\vtheta}\fM$. Then for any $\alpha\in \sR$, the tangent space of $\vtheta'=\fT_{l,s}^{\alpha}(\vtheta)\in \fM'$ is spanned by $\{\fT_{l,s}(\ve_1(\vtheta)),\cdots,\fT_{l,s}(\ve_d(\vtheta)),\fV_{l,s}(\vtheta)\}$. Since $\fT_{l,s}$ is linear and injective, $\{\fT_{l,s}(\ve_i(\vtheta))\}_{i=1}^{d}$ is also a linearly independent set. Moreover, since $\mW^{[l+1]}_{[1:m_{l+1}],m_l+1}=\mzero$ for any vector in parameter space applied with $\fT_{l,s}$, then,  $\{\fT_{l,s}(\ve_1(\vtheta)),\cdots,\fT_{l,s}(\ve_d(\vtheta)),\fV_{l,s}(\vtheta)\}$ are independent if and only if $\fV_{l,s}(\vtheta)\neq \mzero$, i.e., $\mW^{[l+1]}_{[1:m_{l+1}],s}\neq \mzero$.
% \end{proof}

% \begin{definition}[critical submanifold]
% A critical submanifold of a DNN B associated with critical point of a smaller DNN A is defined as  all possible $\vtheta_B$'s obtained through criticality invariant transform from $\vtheta_A$.
% \end{definition}

\begin{theorem}[\textbf{degeneracy of embedded critical points}] \label{lem:multi-criti}
    Consider two $L$-layer ($L\geq 2$) fully-connected neural networks $\mathrm{NN}_A(\{m_l\}_{l=0}^{L})$ and $\mathrm{NN}_B(\{m'_l\}_{l=0}^{L})$ which is $K$-neuron wider than $\mathrm{NN}_A$. Suppose that the critical point $\vtheta_A=(\mW^{[1]},\vb^{[1]},\cdots,\mW^{[L]},\vb^{[L]})$ satisfies $\mW^{[l]}\neq \mzero$ for each layer $l\in[L]$. Then the parameters $\vtheta_A$ of $\mathrm{NN}_A$ can be critically embedded to a $K$-dimensional critical affine subspace $\fM_B=\{\vtheta_B+\sum_{i=1}^{K}\alpha_i \vv_i|\alpha_i\in \sR \}$ of loss landscape of $\mathrm{NN}_B$. Here $\vtheta_B=(\prod_{i=1}^{K} \fT_{l_i,s_i})(\vtheta_A)$ and $\vv_i=\fT_{l_K,s_K}\cdots \fV_{l_i,s_i} \cdots\fT_{l_1,s_1}\vtheta_A$. 
    % each $(l_i,s_i)$ is chosen from indices of the non-silent neurons of $\vtheta_A$, 

% a $K$-neuron wider DNN B such that 

% and any larger $L$-th layer network B of width $m_0-m'_1-\cdots-m'_{L-1}-m_{L}$, where $m'_i\geq m_i$ for $i\in [L-1]$ and $K=\sum_{i=1}^{L-1}(m'_i-m_i)\geq 1$.
% For any length-$K$ sequence $\{l_i\in[L-1]\}_{i=1}^{K}$, $\{s_i\in[m_{l_i}]\}_{i=1}^{K}$, where $\#\{l_i=l\}=m'_l-m_l$,
% $\vtheta_B=(\prod_{i=1}^{K} \fT_{l_i,s_i})(\vtheta_A)$, 
% $\vv_i=\fT_{l_K,s_K}\cdots \fV_{l_i,s_i} \cdots\fT_{l_1,s_1}\vtheta_A$, then 
% $\fM_B=\{\vtheta_B+\sum_{i=1}^{K}\alpha_i \vv_i|\alpha_i\in \sR \}$ is a critical submanifold (affine subspace) of B with invariant output, loss and feature comparing to that of $\vtheta_A$. Moreover, if $(\mW^{[l_i+1]})_{,s_i}\neq \mzero$ for any $i\in[K]$, then $\fM_B$ is $K$-dimensional, otherwise is $\#\{(\mW^{[l_i+1]})_{,s_i}\neq \mzero|i\in[K]\}$-dimensional.
\end{theorem}

Note that  neuron-index permutation among the same layer is a trivial criticality invariant transform. More discussions about it, specifically for NNs of homogeneous activation functions like ReLU, can be found in Section A.1 in Appendix. 
% Example Subsection

% \begin{lemma}[connectedness of critical points]
% Given any two $\vtheta_{B_0}=\fT(\vtheta_A;l_0,s_0,\alpha_0)$ and $\vtheta_{B_1}=\fT(\vtheta_A;l_0,s_1,\alpha_1)$, there exist a path between  $\vtheta_{B_0}$ and $\vtheta_{B_1}$ such that each point on the path is critical.
% \end{lemma}

% \begin{proof}
% ...
% \end{proof}

% \subsection{loss landscape of a two-layer neural network for synthetic data}
% \subsection{\textcolor{red}{critical affine subspaces} of DNNs for real data}
% \subsection{transition from local minima to saddle points through critical embedding}

\section{Conclusion and discussion} \label{sec:dis}
In this work, we prove an embedding principle that loss landscape of a DNN \emph{contains} all critical points of all the narrower DNNs. This embedding principle unravels
wide existence of highly degenerate critical points with low complexity in the loss landscape of a wide DNN, i.e., critical points with low-complexity output function and degenerate Hessian matrix, embedded from critical points of narrow DNNs. 
With such a loss landscape of DNN, the gradient-based training has the potential of getting attracted or even converging to a low complexity critical point as confirmed by above numerical experiments, which implies a potential implicit regularization towards low-complexity function of nonlinear DNN training dynamics.

Moreover, through critical embedding, a critical point in form of a common non-degenerate local minimum of a narrow DNN not only becomes degenerate in general, but also may become a saddle point as illustrated by numerical experiments. This may be the reason underlying the general easy optimization of wide DNNs observed in practice even beyond the linear/kernel/NTK regime \citep{chen2020dynamical,trager2019gradient,geiger2020disentangling,fort2020deep,luo2021phase}. We will perform more detailed analysis as well as numerical experiments specifically about this minimum-to-saddle transition later.

At the essence, the embedding principle results from the layer structure of a neural network model, which allows arbitrary neuron addition, input weight copying and output weight splitting within each layer. Therefore, though results in this work assume fully-connected NNs, these can be easily extended to other DNN architectures. Considering convolutional neural networks for example, the quantity that corresponds to width of fully-connected NNs is channel. Similar to one-step or multi-step embedding, we can introduce a feature splitting operation, i.e., increase the number of channels by splitting all neurons sharing one convolution kernel with the same $\alpha$, which can be proven to preserve the output function, representation as well as the criticality. Thereby, embedding principle holds in a sense that loss landscape of any CNN contains all critical points of all narrower CNNs whose number of channels in each layer is no more than that of the target CNN.
Currently, depth serves as a preset hyperparameter in our analysis. Whether loss landscape of DNNs of different depth has certain embedding relation for specific DNN architectures such as ResNet is an interesting open problem.

Our embedding principle and experiments in Figs. \ref{fig:syntraining} and \ref{fig:expand_mnist} suggest that whenever training of a wide DNN is stagnated around a critical point, it potentially is embedded from a much narrower DNN. Therefore, many neurons with similar representation can be reduced to one neuron. How we can design efficient pruning algorithm to fully realize this potential and how it is related to existing pruning methods as well as the well-known lottery ticket hypothesis \citep{frankle2018lottery} are important problems for our future research.

We remark that our embedding principle applies for landscape of general loss functions. Although for loss functions like cross entropy, a meaningful finite critical point may not exist as its parameters diverge in general throughout the training, yet it is reasonable to expect that critical embedding may provide us certain approximate critical points from narrow NNs. Of course, how to properly define an approximate critical point is in itself a problem of interest. And we leave this problem for the future study.

Overall, our embedding principle provides the first clear picture about the general structure of critical points of DNN loss landscape, which is fundamental to the theoretical understanding of both training and generalization behavior of DNNs as well as the design of optimization algorithms. Of course, the study of loss landscape of DNN is far from complete. This work serves as a starting point for a novel line of research, which finally leads to an exact and comprehensive theoretical description about loss landscape of DNNs as well as an understanding of its profound impact on training and generalization.

\acksection{This work is sponsored by the National Key R\&D Program of China  Grant No. 2019YFA0709503 (Z. X.), the Shanghai Sailing Program, the Natural Science Foundation of Shanghai Grant No. 20ZR1429000  (Z. X.), the National Natural Science Foundation of China Grant No. 62002221 (Z. X.), the National Natural Science Foundation of China Grant No. 12101401 (T. L.), the National Natural Science Foundation of China Grant No. 12101402 (Y. Z.), Shanghai Municipal of Science and Technology Project Grant No. 20JC1419500 (Y.Z.), Shanghai Municipal of Science and Technology Major Project No. 2021SHZDZX0102, and the HPC of School of Mathematical Sciences and the Student Innovation Center at Shanghai Jiao Tong University.}

\bibliographystyle{elsarticle-num-names}
\bibliography{references}

\begin{thebibliography}{41}
\expandafter\ifx\csname natexlab\endcsname\relax\def\natexlab#1{#1}\fi
\providecommand{\url}[1]{\texttt{#1}}
\providecommand{\href}[2]{#2}
\providecommand{\path}[1]{#1}
\providecommand{\DOIprefix}{doi:}
\providecommand{\ArXivprefix}{arXiv:}
\providecommand{\URLprefix}{URL: }
\providecommand{\Pubmedprefix}{pmid:}
\providecommand{\doi}[1]{\href{http://dx.doi.org/#1}{\path{#1}}}
\providecommand{\Pubmed}[1]{\href{pmid:#1}{\path{#1}}}
\providecommand{\bibinfo}[2]{#2}
\ifx\xfnm\relax \def\xfnm[#1]{\unskip,\space#1}\fi
%Type = Article
\bibitem[{E et~al.(2020)E, Ma, Wojtowytsch, and Wu}]{weinan2020towards}
\bibinfo{author}{W.~E}, \bibinfo{author}{C.~Ma},
  \bibinfo{author}{S.~Wojtowytsch}, \bibinfo{author}{L.~Wu},
\newblock \bibinfo{title}{Towards a mathematical understanding of neural
  network-based machine learning: What we know and what we don’t},
\newblock \bibinfo{journal}{arXiv preprint arXiv:2009.10713}
  (\bibinfo{year}{2020}).
%Type = Article
\bibitem[{Skorokhodov and Burtsev(2019)}]{skorokhodov2019loss}
\bibinfo{author}{I.~Skorokhodov}, \bibinfo{author}{M.~Burtsev},
\newblock \bibinfo{title}{Loss landscape sightseeing with multi-point
  optimization},
\newblock \bibinfo{journal}{arXiv preprint arXiv:1910.03867}
  (\bibinfo{year}{2019}).
%Type = Article
\bibitem[{Breiman(1995)}]{breiman1995reflections}
\bibinfo{author}{L.~Breiman},
\newblock \bibinfo{title}{Reflections after refereeing papers for nips},
\newblock \bibinfo{journal}{The Mathematics of Generalization}
  \bibinfo{volume}{XX} (\bibinfo{year}{1995}) \bibinfo{pages}{11--15}.
%Type = Inproceedings
\bibitem[{Zhang et~al.(2017)Zhang, Bengio, Hardt, Recht, and
  Vinyals}]{zhang2016understanding}
\bibinfo{author}{C.~Zhang}, \bibinfo{author}{S.~Bengio},
  \bibinfo{author}{M.~Hardt}, \bibinfo{author}{B.~Recht},
  \bibinfo{author}{O.~Vinyals},
\newblock \bibinfo{title}{Understanding deep learning requires rethinking
  generalization},
\newblock in: \bibinfo{booktitle}{5th International Conference on Learning
  Representations, {ICLR} 2017, Toulon, France, April 24-26, 2017, Conference
  Track Proceedings}, \bibinfo{publisher}{OpenReview.net},
  \bibinfo{year}{2017}.
%Type = Inproceedings
\bibitem[{Arpit et~al.(2017)Arpit, Jastrzebski, Ballas, Krueger, Bengio,
  Kanwal, Maharaj, Fischer, Courville, Bengio, and
  Lacoste{-}Julien}]{arpit2017closer}
\bibinfo{author}{D.~Arpit}, \bibinfo{author}{S.~Jastrzebski},
  \bibinfo{author}{N.~Ballas}, \bibinfo{author}{D.~Krueger},
  \bibinfo{author}{E.~Bengio}, \bibinfo{author}{M.~S. Kanwal},
  \bibinfo{author}{T.~Maharaj}, \bibinfo{author}{A.~Fischer},
  \bibinfo{author}{A.~C. Courville}, \bibinfo{author}{Y.~Bengio},
  \bibinfo{author}{S.~Lacoste{-}Julien},
\newblock \bibinfo{title}{A closer look at memorization in deep networks},
\newblock in: \bibinfo{editor}{D.~Precup}, \bibinfo{editor}{Y.~W. Teh} (Eds.),
  \bibinfo{booktitle}{Proceedings of the 34th International Conference on
  Machine Learning, {ICML} 2017, Sydney, NSW, Australia, 6-11 August 2017},
  volume~\bibinfo{volume}{70} of \textit{\bibinfo{series}{Proceedings of
  Machine Learning Research}}, \bibinfo{publisher}{{PMLR}},
  \bibinfo{year}{2017}, pp. \bibinfo{pages}{233--242}.
%Type = Inproceedings
\bibitem[{Kalimeris et~al.(2019)Kalimeris, Kaplun, Nakkiran, Edelman, Yang,
  Barak, and Zhang}]{kalimeris2019sgd}
\bibinfo{author}{D.~Kalimeris}, \bibinfo{author}{G.~Kaplun},
  \bibinfo{author}{P.~Nakkiran}, \bibinfo{author}{B.~L. Edelman},
  \bibinfo{author}{T.~Yang}, \bibinfo{author}{B.~Barak},
  \bibinfo{author}{H.~Zhang},
\newblock \bibinfo{title}{{SGD} on neural networks learns functions of
  increasing complexity},
\newblock in: \bibinfo{editor}{H.~M. Wallach}, \bibinfo{editor}{H.~Larochelle},
  \bibinfo{editor}{A.~Beygelzimer}, \bibinfo{editor}{F.~d'Alch{\'{e}}{-}Buc},
  \bibinfo{editor}{E.~B. Fox}, \bibinfo{editor}{R.~Garnett} (Eds.),
  \bibinfo{booktitle}{Advances in Neural Information Processing Systems 32:
  Annual Conference on Neural Information Processing Systems 2019, NeurIPS
  2019, December 8-14, 2019, Vancouver, BC, Canada}, \bibinfo{year}{2019}, pp.
  \bibinfo{pages}{3491--3501}.
%Type = Article
\bibitem[{Jin et~al.(2020)Jin, Lu, Tang, and Karniadakis}]{jin2020quantifying}
\bibinfo{author}{P.~Jin}, \bibinfo{author}{L.~Lu}, \bibinfo{author}{Y.~Tang},
  \bibinfo{author}{G.~E. Karniadakis},
\newblock \bibinfo{title}{Quantifying the generalization error in deep learning
  in terms of data distribution and neural network smoothness},
\newblock \bibinfo{journal}{Neural Networks} \bibinfo{volume}{130}
  (\bibinfo{year}{2020}) \bibinfo{pages}{85--99}.
%Type = Article
\bibitem[{Xu et~al.(2019)Xu, Zhang, and Xiao}]{xu_training_2018}
\bibinfo{author}{Z.-Q.~J. Xu}, \bibinfo{author}{Y.~Zhang},
  \bibinfo{author}{Y.~Xiao},
\newblock \bibinfo{title}{Training behavior of deep neural network in frequency
  domain},
\newblock \bibinfo{journal}{International Conference on Neural Information
  Processing}  (\bibinfo{year}{2019}) \bibinfo{pages}{264--274}.
%Type = Article
\bibitem[{Xu et~al.(2020)Xu, Zhang, Luo, Xiao, and Ma}]{xu2019frequency}
\bibinfo{author}{Z.-Q.~J. Xu}, \bibinfo{author}{Y.~Zhang},
  \bibinfo{author}{T.~Luo}, \bibinfo{author}{Y.~Xiao}, \bibinfo{author}{Z.~Ma},
\newblock \bibinfo{title}{Frequency principle: Fourier analysis sheds light on
  deep neural networks},
\newblock \bibinfo{journal}{Communications in Computational Physics}
  \bibinfo{volume}{28} (\bibinfo{year}{2020}) \bibinfo{pages}{1746--1767}.
%Type = Article
\bibitem[{Rahaman et~al.(2019)Rahaman, Arpit, Baratin, Draxler, Lin, Hamprecht,
  Bengio, and Courville}]{rahaman2018spectral}
\bibinfo{author}{N.~Rahaman}, \bibinfo{author}{D.~Arpit},
  \bibinfo{author}{A.~Baratin}, \bibinfo{author}{F.~Draxler},
  \bibinfo{author}{M.~Lin}, \bibinfo{author}{F.~A. Hamprecht},
  \bibinfo{author}{Y.~Bengio}, \bibinfo{author}{A.~Courville},
\newblock \bibinfo{title}{On the spectral bias of deep neural networks},
\newblock \bibinfo{journal}{International Conference on Machine Learning}
  (\bibinfo{year}{2019}).
%Type = Article
\bibitem[{Luo et~al.(2021)Luo, Xu, Ma, and Zhang}]{luo2021phase}
\bibinfo{author}{T.~Luo}, \bibinfo{author}{Z.-Q.~J. Xu},
  \bibinfo{author}{Z.~Ma}, \bibinfo{author}{Y.~Zhang},
\newblock \bibinfo{title}{Phase diagram for two-layer relu neural networks at
  infinite-width limit},
\newblock \bibinfo{journal}{Journal of Machine Learning Research}
  \bibinfo{volume}{22} (\bibinfo{year}{2021}) \bibinfo{pages}{1--47}.
%Type = Inproceedings
\bibitem[{Chizat and Bach(2018)}]{chizat_global_2018}
\bibinfo{author}{L.~Chizat}, \bibinfo{author}{F.~R. Bach},
\newblock \bibinfo{title}{On the global convergence of gradient descent for
  over-parameterized models using optimal transport},
\newblock in: \bibinfo{booktitle}{NeurIPS 2018}, \bibinfo{year}{2018}.
%Type = Article
\bibitem[{Ma et~al.(2020)Ma, Wu, and E}]{ma2020quenching}
\bibinfo{author}{C.~Ma}, \bibinfo{author}{L.~Wu}, \bibinfo{author}{W.~E},
\newblock \bibinfo{title}{The quenching-activation behavior of the gradient
  descent dynamics for two-layer neural network models},
\newblock \bibinfo{journal}{arXiv preprint arXiv:2006.14450}
  (\bibinfo{year}{2020}).
%Type = Inproceedings
\bibitem[{Keskar et~al.(2017)Keskar, Nocedal, Tang, Mudigere, and
  Smelyanskiy}]{keskar2017large}
\bibinfo{author}{N.~S. Keskar}, \bibinfo{author}{J.~Nocedal},
  \bibinfo{author}{P.~T.~P. Tang}, \bibinfo{author}{D.~Mudigere},
  \bibinfo{author}{M.~Smelyanskiy},
\newblock \bibinfo{title}{On large-batch training for deep learning:
  Generalization gap and sharp minima},
\newblock in: \bibinfo{booktitle}{5th International Conference on Learning
  Representations, ICLR 2017}, \bibinfo{year}{2017}.
%Type = Article
\bibitem[{Wu et~al.(2017)Wu, Zhu et~al.}]{wu2017towards}
\bibinfo{author}{L.~Wu}, \bibinfo{author}{Z.~Zhu}, et~al.,
\newblock \bibinfo{title}{Towards understanding generalization of deep
  learning: Perspective of loss landscapes},
\newblock \bibinfo{journal}{arXiv preprint arXiv:1706.10239}
  (\bibinfo{year}{2017}).
%Type = Article
\bibitem[{He et~al.(2019)He, Huang, and Yuan}]{he2019asymmetric}
\bibinfo{author}{H.~He}, \bibinfo{author}{G.~Huang}, \bibinfo{author}{Y.~Yuan},
\newblock \bibinfo{title}{Asymmetric valleys: Beyond sharp and flat local
  minima},
\newblock \bibinfo{journal}{arXiv preprint arXiv:1902.00744}
  (\bibinfo{year}{2019}).
%Type = Article
\bibitem[{Cooper(2021)}]{cooper2018loss}
\bibinfo{author}{Y.~Cooper},
\newblock \bibinfo{title}{Global minima of overparameterized neural networks},
\newblock \bibinfo{journal}{SIAM Journal on Mathematics of Data Science}
  \bibinfo{volume}{3} (\bibinfo{year}{2021}) \bibinfo{pages}{676--691}.
  \DOIprefix\doi{10.1137/19M1308943}.
%Type = Article
\bibitem[{Sagun et~al.(2016)Sagun, Bottou, and LeCun}]{sagun2016singularity}
\bibinfo{author}{L.~Sagun}, \bibinfo{author}{L.~Bottou},
  \bibinfo{author}{Y.~LeCun},
\newblock \bibinfo{title}{Singularity of the hessian in deep learning},
\newblock \bibinfo{journal}{arXiv preprint arXiv:1611.07476}
  (\bibinfo{year}{2016}).
%Type = Inproceedings
\bibitem[{Choromanska et~al.(2015)Choromanska, Henaff, Mathieu, Arous, and
  LeCun}]{choromanska2015loss}
\bibinfo{author}{A.~Choromanska}, \bibinfo{author}{M.~Henaff},
  \bibinfo{author}{M.~Mathieu}, \bibinfo{author}{G.~B. Arous},
  \bibinfo{author}{Y.~LeCun},
\newblock \bibinfo{title}{The loss surfaces of multilayer networks},
\newblock in: \bibinfo{booktitle}{Artificial intelligence and statistics},
  \bibinfo{organization}{PMLR}, \bibinfo{year}{2015}, pp.
  \bibinfo{pages}{192--204}.
%Type = Inproceedings
\bibitem[{Jacot et~al.(2018)Jacot, Hongler, and Gabriel}]{jacot_neural_2018}
\bibinfo{author}{A.~Jacot}, \bibinfo{author}{C.~Hongler},
  \bibinfo{author}{F.~Gabriel},
\newblock \bibinfo{title}{Neural tangent kernel: Convergence and generalization
  in neural networks},
\newblock in: \bibinfo{booktitle}{NeurIPS 2018}, \bibinfo{year}{2018}.
%Type = Inproceedings
\bibitem[{Arora et~al.(2019)Arora, Du, Hu, Li, Salakhutdinov, and
  Wang}]{arora2019exact}
\bibinfo{author}{S.~Arora}, \bibinfo{author}{S.~S. Du},
  \bibinfo{author}{W.~Hu}, \bibinfo{author}{Z.~Li},
  \bibinfo{author}{R.~Salakhutdinov}, \bibinfo{author}{R.~Wang},
\newblock \bibinfo{title}{On exact computation with an infinitely wide neural
  net},
\newblock in: \bibinfo{booktitle}{NeurIPS 2019}, \bibinfo{year}{2019}.
%Type = Inproceedings
\bibitem[{Zhang et~al.(2020)Zhang, Xu, Luo, and Ma}]{zhang_type_2019}
\bibinfo{author}{Y.~Zhang}, \bibinfo{author}{Z.-Q.~J. Xu},
  \bibinfo{author}{T.~Luo}, \bibinfo{author}{Z.~Ma},
\newblock \bibinfo{title}{A type of generalization error induced by
  initialization in deep neural networks},
\newblock in: \bibinfo{booktitle}{MSML 2020}, \bibinfo{year}{2020}.
%Type = Inproceedings
\bibitem[{Du et~al.(2019)Du, Lee, Li, Wang, and Zhai}]{du2019gradient}
\bibinfo{author}{S.~Du}, \bibinfo{author}{J.~Lee}, \bibinfo{author}{H.~Li},
  \bibinfo{author}{L.~Wang}, \bibinfo{author}{X.~Zhai},
\newblock \bibinfo{title}{Gradient descent finds global minima of deep neural
  networks},
\newblock in: \bibinfo{booktitle}{International Conference on Machine
  Learning}, \bibinfo{organization}{PMLR}, \bibinfo{year}{2019}, pp.
  \bibinfo{pages}{1675--1685}.
%Type = Article
\bibitem[{Zou et~al.(2018)Zou, Cao, Zhou, and Gu}]{zou2018stochastic}
\bibinfo{author}{D.~Zou}, \bibinfo{author}{Y.~Cao}, \bibinfo{author}{D.~Zhou},
  \bibinfo{author}{Q.~Gu},
\newblock \bibinfo{title}{Stochastic gradient descent optimizes
  over-parameterized deep relu networks},
\newblock \bibinfo{journal}{arXiv preprint arXiv:1811.08888}
  (\bibinfo{year}{2018}).
%Type = Inproceedings
\bibitem[{Allen-Zhu et~al.(2019)Allen-Zhu, Li, and Song}]{allen2019convergence}
\bibinfo{author}{Z.~Allen-Zhu}, \bibinfo{author}{Y.~Li},
  \bibinfo{author}{Z.~Song},
\newblock \bibinfo{title}{A convergence theory for deep learning via
  over-parameterization},
\newblock in: \bibinfo{booktitle}{International Conference on Machine
  Learning}, \bibinfo{organization}{PMLR}, \bibinfo{year}{2019}, pp.
  \bibinfo{pages}{242--252}.
%Type = Article
\bibitem[{E et~al.(2019)E, Ma, and Wu}]{E2019comparative}
\bibinfo{author}{W.~E}, \bibinfo{author}{C.~Ma}, \bibinfo{author}{L.~Wu},
\newblock \bibinfo{title}{A comparative analysis of the optimization and
  generalization property of two-layer neural network and random feature models
  under gradient descent dynamics},
\newblock \bibinfo{journal}{arXiv preprint arXiv:1904.04326}
  (\bibinfo{year}{2019}).
%Type = Article
\bibitem[{Mei et~al.(2018)Mei, Montanari, and Nguyen}]{mei_mean_2018}
\bibinfo{author}{S.~Mei}, \bibinfo{author}{A.~Montanari},
  \bibinfo{author}{P.-M. Nguyen},
\newblock \bibinfo{title}{A mean field view of the landscape of two-layer
  neural networks},
\newblock \bibinfo{journal}{Proceedings of the National Academy of Sciences}
  \bibinfo{volume}{115} (\bibinfo{year}{2018}) \bibinfo{pages}{E7665--E7671}.
%Type = Inproceedings
\bibitem[{Rotskoff and Vanden{-}Eijnden(2018)}]{rotskoff_parameters_2018}
\bibinfo{author}{G.~M. Rotskoff}, \bibinfo{author}{E.~Vanden{-}Eijnden},
\newblock \bibinfo{title}{Parameters as interacting particles: long time
  convergence and asymptotic error scaling of neural networks},
\newblock in: \bibinfo{booktitle}{NeurIPS 2018}, \bibinfo{year}{2018}.
%Type = Article
\bibitem[{Sirignano and Spiliopoulos(2020)}]{sirignano_mean_2020}
\bibinfo{author}{J.~Sirignano}, \bibinfo{author}{K.~Spiliopoulos},
\newblock \bibinfo{title}{Mean field analysis of neural networks: {A} central
  limit theorem},
\newblock \bibinfo{journal}{Stochastic Processes and their Applications}
  \bibinfo{volume}{130} (\bibinfo{year}{2020}) \bibinfo{pages}{1820--1852}.
%Type = Article
\bibitem[{Goldt et~al.(2020)Goldt, M{\'e}zard, Krzakala, and
  Zdeborov{\'a}}]{goldt2020modeling}
\bibinfo{author}{S.~Goldt}, \bibinfo{author}{M.~M{\'e}zard},
  \bibinfo{author}{F.~Krzakala}, \bibinfo{author}{L.~Zdeborov{\'a}},
\newblock \bibinfo{title}{Modeling the influence of data structure on learning
  in neural networks: The hidden manifold model},
\newblock \bibinfo{journal}{Physical Review X} \bibinfo{volume}{10}
  (\bibinfo{year}{2020}) \bibinfo{pages}{041044}.
%Type = Article
\bibitem[{He et~al.(2020)He, Wang, Shi, Lyu, and Tu}]{he2020assessing}
\bibinfo{author}{S.~He}, \bibinfo{author}{X.~Wang}, \bibinfo{author}{S.~Shi},
  \bibinfo{author}{M.~R. Lyu}, \bibinfo{author}{Z.~Tu},
\newblock \bibinfo{title}{Assessing the bilingual knowledge learned by neural
  machine translation models},
\newblock \bibinfo{journal}{arXiv preprint arXiv:2004.13270}
  (\bibinfo{year}{2020}).
%Type = Article
\bibitem[{Mingard et~al.(2019)Mingard, Skalse, Valle-P{\'e}rez,
  Mart{\'\i}nez-Rubio, Mikulik, and Louis}]{mingard2019neural}
\bibinfo{author}{C.~Mingard}, \bibinfo{author}{J.~Skalse},
  \bibinfo{author}{G.~Valle-P{\'e}rez},
  \bibinfo{author}{D.~Mart{\'\i}nez-Rubio}, \bibinfo{author}{V.~Mikulik},
  \bibinfo{author}{A.~A. Louis},
\newblock \bibinfo{title}{Neural networks are a priori biased towards boolean
  functions with low entropy},
\newblock \bibinfo{journal}{arXiv preprint arXiv:1909.11522}
  (\bibinfo{year}{2019}).
%Type = Article
\bibitem[{Zhang et~al.(2021)Zhang, Li, Zhang, Luo, and Xu}]{zhang2021embedding}
\bibinfo{author}{Y.~Zhang}, \bibinfo{author}{Y.~Li},
  \bibinfo{author}{Z.~Zhang}, \bibinfo{author}{T.~Luo}, \bibinfo{author}{Z.~J.
  Xu},
\newblock \bibinfo{title}{Embedding principle: a hierarchical structure of loss
  landscape of deep neural networks},
\newblock \bibinfo{journal}{arXiv preprint arXiv:2111.15527}
  (\bibinfo{year}{2021}).
%Type = Article
\bibitem[{Fukumizu et~al.(2019)Fukumizu, Yamaguchi, Mototake, and
  Tanaka}]{fukumizu2019semi}
\bibinfo{author}{K.~Fukumizu}, \bibinfo{author}{S.~Yamaguchi},
  \bibinfo{author}{Y.-i. Mototake}, \bibinfo{author}{M.~Tanaka},
\newblock \bibinfo{title}{Semi-flat minima and saddle points by embedding
  neural networks to overparameterization},
\newblock \bibinfo{journal}{Advances in Neural Information Processing Systems}
  \bibinfo{volume}{32} (\bibinfo{year}{2019}) \bibinfo{pages}{13868--13876}.
%Type = Inproceedings
\bibitem[{Simsek et~al.(2021)Simsek, Ged, Jacot, Spadaro, Hongler, Gerstner,
  and Brea}]{csimcsek2021geometry}
\bibinfo{author}{B.~Simsek}, \bibinfo{author}{F.~Ged},
  \bibinfo{author}{A.~Jacot}, \bibinfo{author}{F.~Spadaro},
  \bibinfo{author}{C.~Hongler}, \bibinfo{author}{W.~Gerstner},
  \bibinfo{author}{J.~Brea},
\newblock \bibinfo{title}{Geometry of the loss landscape in overparameterized
  neural networks: Symmetries and invariances},
\newblock in: \bibinfo{booktitle}{Proceedings of the 38th International
  Conference on Machine Learning}, \bibinfo{publisher}{PMLR},
  \bibinfo{year}{2021}, pp. \bibinfo{pages}{9722--9732}.
%Type = Article
\bibitem[{Fisher(1936)}]{fisher1936use}
\bibinfo{author}{R.~A. Fisher},
\newblock \bibinfo{title}{The use of multiple measurements in taxonomic
  problems},
\newblock \bibinfo{journal}{Annals of eugenics} \bibinfo{volume}{7}
  (\bibinfo{year}{1936}) \bibinfo{pages}{179--188}.
%Type = Inproceedings
\bibitem[{Chen et~al.(2020)Chen, Rotskoff, Bruna, and
  Vanden-Eijnden}]{chen2020dynamical}
\bibinfo{author}{Z.~Chen}, \bibinfo{author}{G.~M. Rotskoff},
  \bibinfo{author}{J.~Bruna}, \bibinfo{author}{E.~Vanden-Eijnden},
\newblock \bibinfo{title}{A dynamical central limit theorem for shallow neural
  networks},
\newblock in: \bibinfo{booktitle}{NeurIPS}, \bibinfo{year}{2020}.
%Type = Inproceedings
\bibitem[{Trager et~al.(2019)Trager, Silva, Panozzo, Zorin, and
  Bruna}]{trager2019gradient}
\bibinfo{author}{M.~Trager}, \bibinfo{author}{C.~Silva},
  \bibinfo{author}{D.~Panozzo}, \bibinfo{author}{D.~Zorin},
  \bibinfo{author}{J.~Bruna},
\newblock \bibinfo{title}{Gradient dynamics of shallow univariate relu
  networks},
\newblock in: \bibinfo{booktitle}{Proceedings of the 33rd International
  Conference on Neural Information Processing Systems}, \bibinfo{year}{2019},
  pp. \bibinfo{pages}{8378--8387}.
%Type = Article
\bibitem[{Geiger et~al.(2020)Geiger, Spigler, Jacot, and
  Wyart}]{geiger2020disentangling}
\bibinfo{author}{M.~Geiger}, \bibinfo{author}{S.~Spigler},
  \bibinfo{author}{A.~Jacot}, \bibinfo{author}{M.~Wyart},
\newblock \bibinfo{title}{Disentangling feature and lazy training in deep
  neural networks},
\newblock \bibinfo{journal}{Journal of Statistical Mechanics: Theory and
  Experiment} \bibinfo{volume}{2020} (\bibinfo{year}{2020})
  \bibinfo{pages}{113301}.
%Type = Article
\bibitem[{Fort et~al.(2020)Fort, Dziugaite, Paul, Kharaghani, Roy, and
  Ganguli}]{fort2020deep}
\bibinfo{author}{S.~Fort}, \bibinfo{author}{G.~K. Dziugaite},
  \bibinfo{author}{M.~Paul}, \bibinfo{author}{S.~Kharaghani},
  \bibinfo{author}{D.~M. Roy}, \bibinfo{author}{S.~Ganguli},
\newblock \bibinfo{title}{Deep learning versus kernel learning: an empirical
  study of loss landscape geometry and the time evolution of the neural tangent
  kernel},
\newblock \bibinfo{journal}{Advances in Neural Information Processing Systems}
  \bibinfo{volume}{33} (\bibinfo{year}{2020}).
%Type = Inproceedings
\bibitem[{Frankle and Carbin(2018)}]{frankle2018lottery}
\bibinfo{author}{J.~Frankle}, \bibinfo{author}{M.~Carbin},
\newblock \bibinfo{title}{The lottery ticket hypothesis: Finding sparse,
  trainable neural networks},
\newblock in: \bibinfo{booktitle}{International Conference on Learning
  Representations}, \bibinfo{year}{2018}.

\end{thebibliography}

\newpage
\appendix
\section{Appendix}
%put appendix here

\begin{lemma*}[\textbf{Lemma 1}] 
    Given a $L$-layer ($L\geq 2$) fully-connected neural network with width $(m_0,\ldots,m_{L})$, for any network parameters    $\vtheta=(\mW^{[1]},\vb^{[1]},\cdots,\mW^{[L]},\vb^{[L]})$ and for any $l\in[L-1]$, $s\in[m_l]$, we have the expressions for $\vtheta':=\fT_{l,s}^{\alpha}(\vtheta)$ (see Fig.~\ref{fig:F-Z} for an illustration)\\
    (i) feature vectors in $\vF_{\vtheta'}$: $\vf^{[l']}_{\vtheta'}=\vf^{[l']}_{\vtheta}$, $l'\neq l$ and $\vf^{[l]}_{\vtheta'}=\left[(\vf_{\vtheta}^{[l]})^\T,(\vf_{\vtheta}^{[l]})_s\right]^\T$;
    
    (ii) feature gradients in $\vG_{\vtheta'}$: $\vg^{[l']}_{\vtheta'}=\vg^{[l']}_{\vtheta}$, $l'\neq l$ and $\vg^{[l]}_{\vtheta'}=
    \left[(\vg^{[l]}_{\vtheta})^\T,(\vg^{[l]}_{\vtheta})_s\right]^\T$;
    
    (iii) error vectors in $\vZ_{\vtheta'}$: 
    $\vz^{[l']}_{\vtheta'}=\vz^{[l']}_{\vtheta}$, $l'\neq l$\\
    and $\vz^{[l]}_{\vtheta'}=
    \left[ (\vz_{\vtheta}^{[l]})^\T_{[1:s-1]},(1-\alpha)(\vz_{\vtheta}^{[l]})_s, (\vz_{\vtheta}^{[l]})^\T_{[s+1:m_l]},\alpha (\vz_{\vtheta}^{[l]})_s \right]^\T$.
\end{lemma*}

\begin{proof}
    (i) By the construction of $\vtheta'$, it is clear that $\vf_{\vtheta'}^{[l']}=\vf_{\vtheta}^{[l']}$ for any $l'\in[l-1]$. Then
    \begin{align}
        \vf^{[l]}_{\vtheta'} 
        &= \sigma\circ \left(\left[ {\begin{array}{cc}
        \mW^{[l]} \\
        \mW^{[l]}_{s,[1:m_{l-1}]} \\
        \end{array} } \right] \vf^{[l-1]}_{\vtheta}+
        \left[ {\begin{array}{cc}
        \vb^{[l]} \\
        \vb^{[l]}_s \\
        \end{array} } \right]
        \right)
        = \left[ {\begin{array}{cc}
        \vf^{[l]}_{\vtheta} \\
        (\vf^{[l]}_{\vtheta})_s \\
        \end{array} } \right].
    \end{align}
    Note that
    \begin{equation*}
        \alpha\left[\mzero_{m_{l+1}\times (s-1)},-\mW^{[l+1]}_{[1:m_{l+1}],s},\mzero_{m_{l+1}\times (m_{l}-s)},\mW^{[l+1]}_{[1:m_{l+1}],s}\right]\left[ {\begin{array}{cc}
        \vf^{[l]}_{\vtheta} \\
        (\vf^{[l]}_{\vtheta})_s \\
        \end{array} } \right]=\mzero_{m_{l+1}\times1}.
    \end{equation*}
    Thus
    \begin{align}
        \vf^{[l+1]}_{\vtheta'} 
        % &=\sigma\circ \left(\left(
        % \left[\mW^{[m_{l+1}]},\mzero_{m_{l+1}\times1}\right]
        % +\alpha \left[\mzero_{m_{l+1}\times (s-1)},-\mW^{[l+1]}_{[1:m_{l+1}],s},\mzero_{m_{l+1}\times (m_{l}-s)},\mW^{[l+1]}_{[1:m_{l+1}],s}\right]\right)
        % \left[ {\begin{array}{cc}
        % \vf^{[l]}_{\vtheta} \\
        % (\vf^{[l]}_{\vtheta})_s \\
        % \end{array} } \right]+
        % \vb^{[l+1]} 
        % \right) \nonumber\\
        &=\sigma\circ \left(
        \left[\mW^{[m_{l+1}]},\mzero_{m_{l+1}\times1}\right]
        \left[ {\begin{array}{cc}
        \vf^{[l]}_{\vtheta} \\
        (\vf^{[l]}_{\vtheta})_s \\
        \end{array} } \right]+\mzero_{m_{l+1}\times1}+
        \vb^{[l+1]} 
        \right) = \vf_{\vtheta}^{[l+1]}.
    \end{align}
    Next, by the construction of $\vtheta'$ again, it is clear that $\vf_{\vtheta'}^{[l']}=\vf_{\vtheta}^{[l']}$ for any $l'\in[l+1:L]$. 
    
    (ii) The results for feature gradients $\vg_{\vtheta'}^{[l']}$, $l'\in[L]$ can be calculated in a similar way.
    
    (iii) 
    By the backpropagation and the above facts in (i), we have $\vz^{[L]}_{\vtheta'}=\nabla\ell(\vf^{[L]}_{\vtheta'},\vy)=\nabla\ell(\vf^{[L]}_{\vtheta},\vy)=\vz^{[L]}_{\vtheta}$.
    
    Recalling the recurrence relation for $l'\in[l+1:L-1]$, then we recursively obtain the following equality for $l'$ from $L-1$ down to $l+1$:
    \begin{equation}
        \vz^{[l']}_{\vtheta'}
        =(\mW^{[l'+1]})^\T \vz^{[l'+1]}_{\vtheta'}\circ \vg^{[l'+1]}_{\vtheta'}
        =(\mW^{[l'+1]})^\T \vz^{[l'+1]}_{\vtheta}\circ \vg^{[l'+1]}_{\vtheta}
        =\vz_{\vtheta}^{[l']}.
    \end{equation}
    
    Next,
    \begin{align}
        \vz^{[l]}_{\vtheta'}
        &=\left(\left[\mW^{[m_{l+1}]},\mzero_{m_{l+1}\times1}\right]
        +\alpha \left[\mzero_{m_{l+1}\times (s-1)},-\mW^{[l+1]}_{[1:m_{l+1}],s},\mzero_{m_{l+1}\times (m_{l}-s)},\mW^{[l+1]}_{[1:m_{l+1}],s}\right]\right)^\T
        \vz^{[l+1]}_{\vtheta}\circ \vg^{[l+1]}_{\vtheta} \nonumber\\
        &=\left[ {\begin{array}{cc}
        \vz^{[l]}_{\vtheta} \\
        0 \\
        \end{array} } \right]+
        \left[ {\begin{array}{cc}
        \mzero_{m_{l+1}\times(s-1)}\\
        -\alpha(\vz^{[l]}_{\vtheta})_s \\
        \mzero_{m_{l+1}\times(m_l-s)} \\
        \alpha(\vz^{[l]}_{\vtheta})_s
        \end{array} } \right]\nonumber\\
        &=\left[ (\vz_{\vtheta}^{[l]})^\T_{[1:s-1]},(1-\alpha)(\vz_{\vtheta}^{[l]})_s, (\vz_{\vtheta}^{[l]})^\T_{[s+1:m_l]},\alpha (\vz_{\vtheta}^{[l]})_s \right]^\T.
    \end{align}
    Finally, 
    \begin{align}
        \vz^{[l-1]}_{\vtheta'}
        &=\left[(\mW^{[l]})^\T,(\mW^{[l]})_{s,[1:m_{l-1}]}^\T\right] 
        \left(\left[ {\begin{array}{cc}
        \vz^{[l]}_{\vtheta} \\
        0 \\
        \end{array} } \right]+
        \left[ {\begin{array}{cc}
        \mzero_{m_{l+1}\times(s-1)}\\
        -\alpha(\vz^{[l]}_{\vtheta})_s \\
        \mzero_{m_{l+1}\times(m_l-s)} \\
        \alpha(\vz^{[l]}_{\vtheta})_s
        \end{array} } \right]\right)
        \circ \left[ {\begin{array}{cc}
        \vg^{[l]}_{\vtheta} \\
        (\vg^{[l]}_{\vtheta})_s \\
        \end{array} } \right]\nonumber\\
        &=(\mW^{[l]})^\T\vz^{[l]}_{\vtheta}\circ\vg^{[l]}_{\vtheta}+\mzero_{m_{l-1}\times 1}\nonumber\\
        &=\vz_{\vtheta}^{[l-1]}.
    \end{align}
    This with the recurrence relation again leads to $\vz_{\vtheta'}^{[l']}=\vz_{\vtheta}^{[l']}$ for all $l'\in[1:l-1]$.
\end{proof}

\begin{prop*}[\textbf{Proposition 1: one-step embedding preserves network properties}] 
    Given a $L$-layer ($L\geq 2$) fully-connected neural network with width $(m_0,\ldots,m_{L})$, for any network parameters $\vtheta=(\mW^{[1]},\vb^{[1]},\cdots,\mW^{[L]},\vb^{[L]})$ and for any $l\in[L-1]$, $s\in[m_l]$, the following network properties are preserved for $\vtheta'=\fT_{l,s}^{\alpha}(\vtheta)$:
    
    (i) output function is preserved: $f_{\vtheta'}(\vx)=f_{\vtheta}(\vx)$ for all $\vx$;
    
    (ii) empirical risk is preserved: $\RS(\vtheta')=\RS(\vtheta)$;
    
    (iii) the sets of features are preserved:  $\left\{\left(\vf^{[l]}_{\vtheta'}\right)_i\right\}_{i\in[m_{l}+1]}=\left\{\left(\vf^{[l]}_{\vtheta}\right)_i\right\}_{i\in[m_{l}]}$ and\\
    $\left\{\left(\vf^{[l']}_{\vtheta'}\right)_i\right\}_{i\in[m_{l'}]}=\left\{\left(\vf^{[l']}_{\vtheta}\right)_i\right\}_{i\in[m_{l'}]}$ for $l'\in[L]\backslash\{l\}$;
    
    % (iv) criticality is preserved: if $\nabla_{\vtheta}\RS(\vtheta)=\mzero$, then $\nabla_{\vtheta}\RS(\vtheta')=\mzero$.
\end{prop*}

\begin{proof}
    The properties (i)--(iii) are direct consequences of Lemma 1. 
    % The criticality preserving property as the most important property of critical embedding is separately stated and proved below.
\end{proof}

\begin{theorem*}[\textbf{Theorem 1: criticality preserving}]
    Given a $L$-layer ($L\geq 2$) fully-connected neural network with width $(m_0,\ldots,m_{L})$, for any network parameters    $\vtheta=(\mW^{[1]},\vb^{[1]},\cdots,\mW^{[L]},\vb^{[L]})$ and for any $l\in[L-1]$, $s\in[m_l]$, if $\nabla_{\vtheta}\RS(\vtheta)=\mzero$, then $\nabla_{\vtheta}\RS(\vtheta')=\mzero$.
\end{theorem*}

\begin{proof}
    Gradient of loss with respect to network parameters of each layer can be computed from $\vF$, $\vG$, and $\vZ$ as follows
    \begin{align*}
        \nabla_{\mW^{[l']}}R_S(\vtheta)
        &= \nabla_{\mW^{[l']}}\Exp_S \ell(\vf_{\vtheta}(\vx),\vy)=\Exp_S\left(\vz_{\vtheta}^{[l']}\circ \vg^{[l']}_{\vtheta}(\vf_{\vtheta}^{[l'-1]})^\T\right),\\
        \nabla_{\vb^{[l']}}R_S(\vtheta)
        &= \nabla_{\vb^{[l]}}\Exp_S \ell(\vf_{\vtheta}(\vx),\vy)=\Exp_S(\vz_{\vtheta}^{[l']}\circ \vg^{[l']}_{\vtheta}).
    \end{align*}
    
    Then, by Lemma 1, we have $\nabla_{\mW^{[l']}}R_S(\vtheta')=
    \nabla_{\mW^{[l']}}R_S(\vtheta)=\mzero$ for $l'\neq l, l+1$ and $\nabla_{\vb^{[l']}}R_S(\vtheta')=
    \nabla_{\vb^{[l']}}R_S(\vtheta)=\mzero$ for $l'\neq l$. Also, for any $j\in[m_{l+1}]$,$k\in[m_{l}]$, since $(\vz_{\vtheta'}^{[l+1]})_j=(\vz_{\vtheta}^{[l+1]})_j$, 
    $(\vg^{[l+1]}_{\vtheta'})_j=(\vg^{[l+1]}_{\vtheta})_j$, and 
    $(\vf^{[l]}_{\vtheta'})_k=(\vf^{[l]}_{\vtheta})_k$, $(\vf^{[l]}_{\vtheta'})_{m_l+1}=(\vf^{[l]}_{\vtheta})_s$, we obtain
    \begin{align*}
        \nabla_{\mW^{[l+1]}_{j,k}}R_S(\vtheta')
        &= \nabla_{\mW^{[l+1]}_{j,k}}R_S(\vtheta)=0,\\
        \nabla_{\mW^{[l+1]}_{j,m_l+1}}R_S(\vtheta')
        &= \nabla_{\mW^{[l+1]}_{j,s}}R_S(\vtheta)=0.
    \end{align*}
    Similarly, for any $j\in[m_{l}]\backslash \{s\}$,$k\in[m_{l-1}]$, we have
    \begin{align*}
        \nabla_{\mW^{[l]}_{j,k}}R_S(\vtheta')
        &= \nabla_{\mW^{[l]}_{j,k}}R_S(\vtheta)=0,\\
        \nabla_{\vb^{[l]}_{j}}R_S(\vtheta')
        &= \nabla_{\vb^{[l]}_{j}}R_S(\vtheta)=0,\\
        \nabla_{\mW^{[l]}_{s,k}}R_S(\vtheta')
        &= (1-\alpha)\nabla_{\mW^{[l]}_{s,k}}R_S(\vtheta)=0,\\
        \nabla_{\mW^{[l]}_{m_l+1,k}}R_S(\vtheta')
        &= \alpha\nabla_{\mW^{[l]}_{s,k}}R_S(\vtheta)=0,\\
        \nabla_{\vb^{[l]}_{s}}R_S(\vtheta')
        &= (1-\alpha)\nabla_{\vb^{[l]}_{s}}R_S(\vtheta)=0,\\
        \nabla_{\vb^{[l]}_{m_l+1}}R_S(\vtheta')
        &= \alpha\nabla_{\vb^{[l]}_{s}}R_S(\vtheta)=0.
    \end{align*}
    Collecting all the above equalities, we have $\nabla_{\vtheta}\RS(\vtheta')=\mzero$.
\end{proof}

% \begin{lemma*}[\textbf{Lemma 2: increment of the degree of degeneracy}] 
%     Given a $L$-layer ($L\geq 2$) fully-connected neural network with width $(m_0,\ldots,m_{L})$, if there exists $l\in[L-1]$, $s\in[m_l]$, and a $q$-dimensional manifold $\fM$ consisting of critical points of $\RS$ such that for any $\vtheta\in \fM$, $\mW^{[l+1]}_{[1:m_{l+1}],s}\neq \mzero$, then $\fM':=\{\fT_{l,s}^{\alpha}(\vtheta)|\vtheta\in\fM,\alpha\in \sR\}$ is a $(q+1)$-dimensional manifold consists of critical points for the corresponding $L$-layer fully-connected neural network with width $(m_0,\ldots,m_{l-1},m_l+1,m_{l+1},\ldots,m_{L})$.
% \end{lemma*}

% \begin{proof}
%     For any $\vtheta\in \fM$, let $\{\ve_i(\vtheta)\}_{i=1}^q$ be a basis of its tangent space $T_{\vtheta}\fM$. Then for any $\alpha\in \sR$, the tangent space of $\vtheta'=\fT_{l,s}^{\alpha}(\vtheta)\in \fM'$ is spanned by $\{\fT_{l,s}(\ve_1(\vtheta)),\cdots,\fT_{l,s}(\ve_q(\vtheta)),\fV_{l,s}(\vtheta)\}$. Since $\fT_{l,s}$ is linear and injective, $\{\fT_{l,s}(\ve_i(\vtheta))\}_{i=1}^{q}$ is also a linearly independent set. Moreover, since $\mW^{[l+1]}_{[1:m_{l+1}],m_l+1}=\mzero$ for any vector in parameter space applied with $\fT_{l,s}$, then,  $\{\fT_{l,s}(\ve_1(\vtheta)),\cdots,\fT_{l,s}(\ve_q(\vtheta)),\fV_{l,s}(\vtheta)\}$ are independent if and only if $\fV_{l,s}(\vtheta)\neq \mzero$, i.e., $\mW^{[l+1]}_{[1:m_{l+1}],s}\neq \mzero$.
% \end{proof}

\begin{lemma*}[\textbf{Lemma 2: increment of the degree of degeneracy}] 
    Given a $L$-layer ($L\geq 2$) fully-connected neural network with width $(m_0,\ldots,m_{L})$, if there exists $l\in[L-1]$, $s\in[m_l]$, and a $q$-dimensional differential manifold $\fM$ consisting of critical points of $\RS$ such that for any $\vtheta\in \fM$, $\mW^{[l+1]}_{[1:m_{l+1}],s}\neq \mzero$, then $\fM':=\{\fT_{l,s}^{\alpha}(\vtheta)|\vtheta\in\fM,\alpha\in \sR\}$ is a $(d+1)$-dimensional differential manifold consists of critical points for the corresponding $L$-layer fully-connected neural network with width $(m_0,\ldots,m_{l-1},m_l+1,m_{l+1},\ldots,m_{L})$.
\end{lemma*}

\begin{proof}
    For any $\vtheta\in \fM$, let $\{\ve_i(\vtheta)\}_{i=1}^d$ be a basis of its tangent space $T_{\vtheta}\fM$. Then for any $\alpha\in \sR$, the tangent space of $\vtheta'=\fT_{l,s}^{\alpha}(\vtheta)\in \fM'$ is spanned by $\{\fT_{l,s}(\ve_1(\vtheta)),\cdots,\fT_{l,s}(\ve_d(\vtheta)),\fV_{l,s}(\vtheta)\}$. Since $\fT_{l,s}$ is linear and injective, $\{\fT_{l,s}(\ve_i(\vtheta))\}_{i=1}^{d}$ is also a linearly independent set. Moreover, since $\mW^{[l+1]}_{[1:m_{l+1}],m_l+1}=\mzero$ for any vector in parameter space applied with $\fT_{l,s}$, then,  $\{\fT_{l,s}(\ve_1(\vtheta)),\cdots,\fT_{l,s}(\ve_d(\vtheta)),\fV_{l,s}(\vtheta)\}$ are independent if and only if $\fV_{l,s}(\vtheta)\neq \mzero$, i.e., $\mW^{[l+1]}_{[1:m_{l+1}],s}\neq \mzero$.
\end{proof}

\begin{rmk}
    The requirement that $\fM$ is a $q$-dimensional differential manifold can be relaxed to that $\fM$ is a $q$-dimensional topological manifold. In the latter case, $\fM'$ is a $(d+1)$-dimensional topological manifold.
\end{rmk}

% \begin{definition}[critical submanifold]
% A critical submanifold of a DNN B associated with critical point of a smaller DNN A is defined as  all possible $\vtheta_B$'s obtained through criticality invariant transform from $\vtheta_A$.
% \end{definition}

\begin{theorem*}[\textbf{Theorem 2: degeneracy of embedded critical points}] 
    Consider two $L$-layer ($L\geq 2$) fully-connected neural networks $\mathrm{NN}_A(\{m_l\}_{l=0}^{L})$ and $\mathrm{NN}_B(\{m'_l\}_{l=0}^{L})$ which is $K$-neuron wider than $\mathrm{NN}_A$. Suppose that the critical point $\vtheta_A=(\mW^{[1]},\vb^{[1]},\cdots,\mW^{[L]},\vb^{[L]})$ satisfy $\mW^{[l]}\neq \mzero$ for each layer $l\in[L]$. Then the parameters $\vtheta_A$ of $\mathrm{NN}_A$ can be critically embedded to a $K$-dimensional critical affine subspace $\fM_B=\{\vtheta_B+\sum_{i=1}^{K}\alpha_i \vv_i|\alpha_i\in \sR \}$ of loss landscape of $\mathrm{NN}_B$. Here $\vtheta_B=(\prod_{i=1}^{K} \fT_{l_i,s_i})(\vtheta_A)$ and $\vv_i=\fT_{l_K,s_K}\cdots \fV_{l_i,s_i} \cdots\fT_{l_1,s_1}\vtheta_A$. 
\end{theorem*}

\begin{proof}
    The assumption $\mW^{[l]}\neq \mzero$, $l\in[L]$ implies the existence of non-silent neurons, i.e., existing $s\in[m_l]$ such that $\mW^{[l+1]}_{[1:m_{l+1}],s}\neq \mzero$, for any $l\in[L-1]$ with $m'_l>m_l$.
    
    In this proof, we misuse notation and denote $m_l=m_l(\vtheta)$ for the width of the $l$-th layer for any fully-connected neural network with parameters $\vtheta$. For such a general network with parameters $\vtheta$, we introduce the following operator. Given an index set $J$ and for any $l\in[L]$, $s\in[m_l]$, we define 
    \begin{equation*}
        \fV_{l,s, J}(\vtheta)=\left(\mzero_{m_0\times m_1},\cdots,
        \left[\mzero_{m_{l+1}\times (s-1)},-\sum_{j\in  J}\mW^{[l+1]}_{[1:m_{l+1}],j},\mzero_{m_{l+1}\times (m_{l}-s)},\sum_{j\in  J}\mW^{[l+1]}_{[1:m_{l+1}],j}\right],
        \cdots\right).
    \end{equation*}
    Clearly, $\fV_{l,s, J}=\sum_{j\in J}\fV_{l,s,\{j\}}$. If for all $j\in J$, $\mW^{[l]}_{j,[1:m_{l-1}]}=\mW^{[l]}_{s,[1:m_{l-1}]}$, $\mW^{[l+1]}_{[1:m_{l+1}],j}=\beta_j\sum_{j'\in J}\mW^{[l+1]}_{[1:m_{l+1}],j'}$ and $\sum_{j'\in J}\mW^{[l+1]}_{[1:m_{l+1}],j'}\neq \mzero$, then for $\beta_s\neq 0$, we have
    \begin{align*}
        \fT_{l,s, J}^{\alpha}\vtheta
        &=(\fT_{l,s}+\alpha\fV_{l,s, J})\vtheta \\
        &= (\fT_{l,s}+\alpha\sum_{j\in J}\fV_{l,s,j})\vtheta\\
        &=  (\fT_{l,s}+\alpha\sum_{j\in J}\frac{\beta_j}{\beta_s}\fV_{l,s})\vtheta\\
        &= \fT_{l,s}^{\alpha'}\vtheta,
    \end{align*} 
    where $\alpha'=\alpha\sum_{j\in J}\frac{\beta_j}{1-\beta_j}$, is simply the one-step critical embedding. We can extend this to the case of $\beta_s=0$.

    Then, for $ J'= J\cup\{m_l+1\}$, $l'=l$, $s'\in J$,
    % Otherwise, there exist $s'\in J$ such that $\beta_s'\neq0$.
    \begin{align*}
        \fV_{l',s', J'}\fV_{l,s, J}
        &= (\fV_{l,s', J}+\fV_{l,s',m_l+1})\fV_{l,s, J}\\
        &= \fV_{l,s', J}\fV_{l,s, J}+\fV_{l,s',m_l+1}\fV_{l,s, J} \\
        &= \fV_{l,s',s}\fV_{l,s, J}+\fV_{l,s',m_l+1}\fV_{l,s, J} \\
        &= \mzero.
    \end{align*}
    In general, we have
    \begin{align*}
        \fV_{l',s', J'}\prod_{i=1}^{N}\fT_{l'_i,s'_i}\fV_{l,s, J}
        &= (\fV_{l,s', J}+\fV_{l,s',m_l+1})\prod_{i=1}^{N}\fT_{l'_i,s'_i}\fV_{l,s, J}\\
        &= \fV_{l,s', J}\prod_{i=1}^{N}\fT_{l'_i,s'_i}\fV_{l,s, J}+\fV_{l,s',m_l+1}\prod_{i=1}^{N}\fT_{l'_i,s'_i}\fV_{l,s, J} \\
        &= \fV_{l,s',s}\prod_{i=1}^{N}\fT_{l'_i,s'_i}\fV_{l,s, J}+\fV_{l,s',m_l+1}\prod_{i=1}^{N}\fT_{l'_i,s'_i}\fV_{l,s, J} \\
        &= \mzero. \\
    \end{align*}
    For $l'\neq l$ or $s'\notin J$, obviously we have $\fV_{l',s', J'}\fV_{l,s, J}=\mzero$ and $\fV_{l',s', J'}\prod_{i=1}^{N}\fT_{l'_i,s'_i}\fV_{l,s, J}=\mzero$.

    Now we are ready to prove the lemma. Let $ J_i=\{s_i\}\cup\{m_l+\#\{i|l_i=l,i\in[j]\}|l_j=l,s_j=s,j\in[K]\}$, where $\#$ indicates number of elements in a set. Then
    \begin{align*}
        \prod_{i=1}^{K} \fT_{l_i,s_i, J_i}^{\alpha_i} &= \prod_{i=1}^{K} (\fT_{l_i,s_i}+\alpha_i\fV_{l_i,s_i, J_i}) \\
        &= \prod_{i=1}^{K} \fT_{l_i,s_i}+\sum_{i=1}^{K}\alpha_i\fT_{l_K,s_K}\cdots \fV_{l_i,s_i, J_i} \cdots\fT_{l_1,s_1}, \\
        &= \prod_{i=1}^{K} \fT_{l_i,s_i}+\sum_{i=1}^{K}\alpha_i\fT_{l_K,s_K}\cdots \fV_{l_i,s_i} \cdots\fT_{l_1,s_1},
    \end{align*}
    which is a critical embedding for any $[\alpha_i]_{i=1}^{K}\in\sR^K$. This completes the proof.
% \begin{align}
% \fT_{l',s'^*}^{\alpha}\fT_{l,s^*}^{\alpha'}
% &=(\fT_{l',s'}+\alpha\fV_{l',s'^*})(\fT_{l,s}+\alpha\fV_{l,s^*})  \\
% &= \fT_{l',s'}\fT_{l,s}+\alpha\fV_{l',s'^*}\fT_{l,s}+\alpha'\fT_{l',s'}\fV_{l,s^*}+\alpha\alpha'\fV_{l',s'^*}\fV_{l,s^*} \\
% &= \fT_{l',s'}\fT_{l,s}+\alpha\fV_{l',s'^*}\fT_{l,s}+\alpha'\fT_{l',s'^}\fV_{l,s^*}.
% \end{align}

\end{proof}

\begin{figure}
    \centering
    \includegraphics[width=0.5\textwidth]{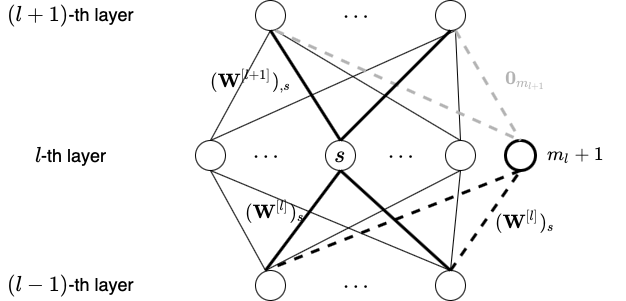}\includegraphics[width=0.5\textwidth]{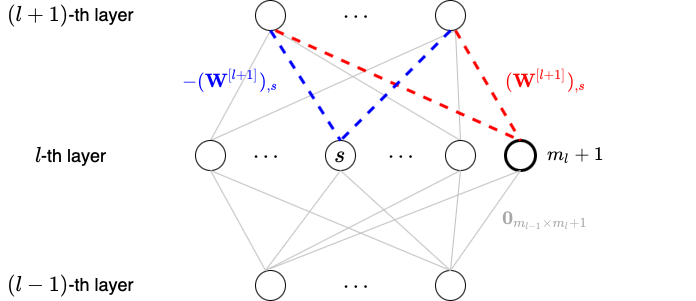}
    \vspace{12 pt}

    \includegraphics[width=0.5\textwidth]{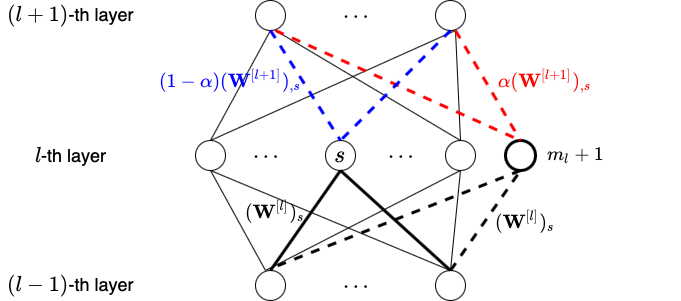}
    \caption{Illustration of $\fT_{l,s}$, $\fV_{l,s}$, and $\fT_{l,s}^{\alpha}$.}
    \label{fig:critical-embedding}
\end{figure}

\begin{figure}
    \centering
    \includegraphics[width=0.7\textwidth]{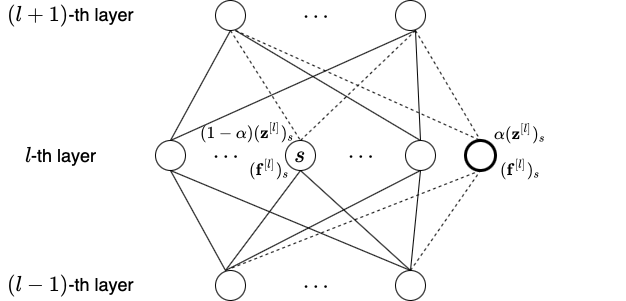}
    \caption{Illustration of $\vF$ and $\vZ$}
    \label{fig:F-Z}
\end{figure}

\subsection{Trivial critical transforms} \label{sec: trivial}
In general, neuron-index permutation among the same layer is a trivial criticality invariant transform because of the layer-wise intrinsic symmetry of DNN models. Therefore, any critical point/manifold may result in multiple "mirror" critical points/manifolds of the loss landscape through all possible permutations. However, this transform does not inform about the degeneracy of critical points/manifolds.

 For any $p$-homogeneous activation function $\sigma$ i.e., $\sigma(\beta x)=\beta^p\sigma(x)$ for any $\beta>0$ and $x\in\sR$, we define for any $l\in[L-1]$, $s\in[m_l]$ the following scaling transform 
 $\vtheta'=\fS_{l,s}^{\beta}(\vtheta)$ ($\beta\neq0$) such that $\mW'^{[l+1]}_{[1:m_{l+1}],s}=\frac{1}{\beta^p}\mW^{[l+1]}_{[1:m_{l+1}],s}$ and $\mW'^{[l]}_{s,[1:m_{l-1}]}=\beta\mW^{[l]}_{s,[1:m_{l-1}]}$,$\vb'^{[l]}_{s}=\beta\vb^{[l]}_{s}$, and all the other entries remain the same. Clearly, this transform is also a critical transform. Moreover, it informs about one more degenerate dimension for each neuron with $\Norm{\mW^{[l+1]}_{[1:m_{l+1}],s}}_2\Norm{\left(\mW^{[l]\T}_{s,[1:m_{l-1}]},\vb^{[l]}_{s}\right)^\T}_2\neq 0$. This critical scaling transform is trivial in a sense that it is an obvious result of the cross-layer scaling preserving intrinsic to each DNN of homogeneous activation function, not relevant to cross-width landscape similarity between DNNs we focus on.

\section{Details of experiments}\label{sec:Appdenix_B}

For the 1D fitting experiments (Figs. 1, 3(a), 4), we use tanh as the activation function, mean squared error (MSE) as the loss function. We use the full-batch gradient descent with learning rate 0.005 to train NNs for 300000 epochs. The initial distribution of all parameters follows a normal distribution with a mean of 0 and a variance of $\frac{1}{m^3}$.

For the iris classification experiment (Fig. 3(b)), we use sigmoid as the activation function, MSE as the loss function. We use the default Adam optimizer of full batch with learning rate 0.02 to train for 500000 epochs. The initial distribution of all parameters follows a normal distribution with mean $0$ and variance $\frac{1}{m^6}$.

For the experiment of MNIST classification (Fig. 5), we use ReLU as the activation function, MSE as the loss function. We also use the default Adam optimizer of full batch with learning rate 0.00003 to train for 100000 epochs. The initial distribution of all parameters follows a normal distribution with mean $0$ and variance $\frac{1}{m^6}$.

To obtain the empirical diagram in Fig. 4, we run 200 trials each for width-1, width-2 and width-3 tanh NNs with variance of initial parameters $\frac{1}{m^3}$ ($m=1,2,3$) for $300000$ epochs. Then we find all parameters with gradient less than $10^{-10}$, which we define as empirical critical points, throughout the training in total $600$ trajectories. Next, we cluster them based on their loss values, output functions, input parameters of neurons and only $4$ different cases arises after excluding the trivial case of constant zero output. Their output functions are shown in the figure.

Remark that, although Figs. 1 and 5 are case studies each based on a random trial, similar phenomenon can be easily observed as long as the initialization variance is properly small, i.e., far from the linear/kernel/NTK regime.

\end{document}